\documentclass[journal]{IEEEtran}

\usepackage{graphicx,epsfig,subfigure,epstopdf}
\usepackage{amsfonts,amssymb,amsmath,latexsym}
\usepackage[ruled,linesnumbered]{algorithm2e}
\usepackage{algpseudocode}
\usepackage{amssymb}
\usepackage{gensymb}
\usepackage{bm}

\usepackage{multirow}
\usepackage{array}
\usepackage{float}
\usepackage{color}
\usepackage{url}
\usepackage{hyperref}
\usepackage[square,comma,numbers,sort&compress]{natbib}
\usepackage{booktabs}
\usepackage{enumitem}

\usepackage{amsthm}
\newtheorem{theorem}{Theorem}
\usepackage{comment}
\usepackage{hyperref}

\DeclareMathOperator*{\argmax}{arg\,max}

\begin{document}
\title{Magnetic Field-Based Reward Shaping for Goal-Conditioned Reinforcement Learning}
\author{Hongyu Ding,~Yuanze Tang,~Qing Wu,~Bo Wang,~\textit{Member, IEEE},~Chunlin~Chen,~\textit{Senior Member, IEEE},~\newline Zhi Wang,~\textit{Member, IEEE}
	
\thanks{This work is accepted by \textit{IEEE-CAA Journal of Automatica Sinica}, 2023, DOI: 10.1109/JAS.2023.123477.}

\thanks{This work was supported in part by the National Natural Science Foundation of China under Grant 62006111 and Grant 62073160, and in part by the Natural Science Foundation of Jiangsu Province of China under Grant BK20200330. \textit{(Corresponding author: Zhi Wang.)}}
\thanks{H. Ding, B. Wang, C. Chen, and Z. Wang are with the Department of Control Science and Intelligence Engineering, School of Management and Engineering, Nanjing University, Nanjing 210093, China (email: hongyuding@smail.nju.edu.cn; \{bowangsme, clchen, zhiwang\}@nju.edu.cn).}
\thanks{Y. Tang and Q. Wu are with the Department of Power Engineering and Process Machinery, School of Mechanical and Power Engineering, East China University of Science and Technology, Shanghai 200237, China (email: yztang@mail.ecust.edu.cn; qwu@ecust.edu.cn).}
}

\maketitle

\begin{abstract}
Goal-conditioned reinforcement learning (RL) is an interesting extension of the traditional RL framework, where the dynamic environment and reward sparsity can cause conventional learning algorithms to fail.
Reward shaping is a practical approach to improving sample efficiency by embedding human domain knowledge into the learning process.
Existing reward shaping methods for goal-conditioned RL are typically built on distance metrics with a linear and isotropic distribution, which may fail to provide sufficient information about the ever-changing environment with high complexity.
This paper proposes a novel magnetic field-based reward shaping (MFRS) method for goal-conditioned RL tasks with dynamic target and obstacles.
Inspired by the physical properties of magnets, we consider the target and obstacles as permanent magnets and establish the reward function according to the intensity values of the magnetic field generated by these magnets.
The nonlinear and anisotropic distribution of the magnetic field intensity can provide more accessible and conducive information about the optimization landscape, thus introducing a more sophisticated magnetic reward compared to the distance-based setting.
Further, we transform our magnetic reward to the form of potential-based reward shaping by learning a secondary potential function concurrently to ensure the optimal policy invariance of our method.
Experiments results in both simulated and real-world robotic manipulation tasks demonstrate that MFRS outperforms relevant existing methods and effectively improves the sample efficiency of RL algorithms in goal-conditioned tasks with various dynamics of the target and obstacles.

\end{abstract}

\begin{IEEEkeywords}
Dynamic environments, goal-conditioned reinforcement learning, magnetic field, reward shaping.
\end{IEEEkeywords}

\section{Introduction}~\label{introduction}
\IEEEPARstart{R}{einforcement} learning (RL)~\citep{sutton2018reinforcement} is a general optimization framework of how an autonomous active agent learns an optimal behavior policy that maximizes the cumulative reward while interacting with its environment in a trial-and-error manner.
Traditional RL algorithms, such as dynamic programming~\citep{bellman1966dynamic}, Monte Carlo methods~\citep{tesauro1996line}, and temporal difference learning~\citep{watkins1992q}, have been widely applied to Markov decision processes (MDPs) with discrete state and action space~\citep{wang2019incremental, chen2013fidelity, luo2016model, zhang2016fmrq}. 
In recent years, with advances in deep learning, RL combined with neural networks has led to great success in high-dimensional applications with continuous space~\citep{zhu2016iterative, xu2007kernel, zheng2018deep, yu2018reusable, wang2022dirichlet}, 
such as video games~\citep{mnih2015human, vinyals2019grandmaster, pan2018multisource}, the game of Go~\citep{silver2017mastering}, robotics~\citep{duan2016benchmarking,wang2022lifelong}, and autonomous driving~\citep{li2019reinforcement}.

While RL has achieved significant success in various benchmark domains, its deployment in real-world applications is still limited since the environment is usually dynamic, which can cause conventional learning algorithms to fail.
Goal-conditioned reinforcement learning~\citep{kaelbling1993learning, schaul2015universal, andrychowicz2017hindsight} is a typical dynamic environment setting where the agent is required to reach ever-changing goals that vary for each episode by learning a general goal-conditioned policy.
The universal value function approximator (UVFA)~\citep{schaul2015universal} makes it possible by letting the Q-function depend on not only a state-action pair but also a goal.
\textcolor{black}{In this paper, we treat the obstacle as an opposite type of the goal in goal-conditioned RL, i.e., ``anti-goal" that the agent should avoid.}
Following UVFA~\citep{schaul2015universal}, we can naturally give a positive reward when the agent reaches the target, a negative one when hitting on the obstacles, and 0 otherwise.
\textcolor{black}{However, the resulting sparse rewards make it difficult for the agent to obtain valid information and lead to poor sample efficiency~\citep{marom2018belief, dann2019deriving}.
Especially in the dynamic environment, the problem is more severe since the inherently sparse reward is constantly changing over time~\citep{andrychowicz2017hindsight}.}

Reward shaping~\citep{ng1999policy, wiewiora2003principled, devlin2012dynamic, harutyunyan2015expressing} effectively addresses the sample efficiency problem in sparse reward tasks without changing the original optimal policy.
\textcolor{black}{It adds an extra shaping reward to the original environment reward to form a new shaped reward that the agent applies to update policy, offering additional dense information about the environment with human domain knowledge.}
Potential-based reward shaping (PBRS)~\citep{ng1999policy} defines the strict form of the shaping reward function as the difference between the potential functions of the successor and the current state. It ensures that the optimal policy learned from the shaped reward is consistent with the one learned from the original reward, also known as the optimal policy invariance property.

The shaping reward in existing methods is typically set according to a distance metric for goal-conditioned RL, e.g., the Euclidean distance between the agent and goal~\citep{trott2019keeping}.
However, due to the linear and isotropic properties of the distance metrics, it may fail to provide adequate information about a dynamic environment with high complexity. 
Specifically, the derivative of the linear distance-based shaping reward is a constant, which cannot provide higher-order information on the optimization landscape.
In this case, the agent only knows ``where'' to go, rather than ``how'' to go.
On the other hand, the isotropic property also exacerbates this problem. 
Considering a 2D plain with an initial state and a target, the points with the same Euclidean distance from the target will share a same value of potential or shaping reward under distance-based setting.
However, the points between the initial state and target should be given higher values since the optimal policy is supposed to be the steepest path from the initial state to the target.
Consequently, it is necessary to develop a more sophisticated reward shaping method that can provide sufficient and conducive information for goal-conditioned tasks.

Intuitively, the agent should be attracted to the target and repelled by the obstacles, which is similar to the physical phenomenon that magnets of the same polarity repel each other while those of different polarities attract each other.
Inspired by this, we propose a magnetic field-based reward shaping (MFRS) method for goal-conditioned RL tasks. We consider the dynamic target and obstacles as permanent magnets, and build the reward function according to the intensity values of the magnetic field generated by these magnets.
The nonlinear and anisotropic properties of the generated magnetic field provide a sophisticated reward function that carries more accessible and sufficient information about the optimization landscape than in the distance-based setting.
Besides, we use several normalization techniques to unify the calculated values of magnetic field intensity and prevent the value exploding problem.
Then, we define the form of our magnetic reward function and transform it into a potential-based one by learning a secondary potential function concurrently to ensure the optimal policy invariance of our method.

In summary, our main contributions are as follows.
\begin{enumerate}
\item We propose a novel magnetic field-based reward shaping (MFRS) method for goal-conditioned RL tasks with dynamic target and obstacles.
\item We build a 3-D simulated robotic manipulation platform and verify the superiority of MFRS on a set of tasks with various goal dynamics.
\item We apply MFRS to a physical robot in the real-world environment and demonstrate the effectiveness of our method.
\end{enumerate}
	
The rest of this paper is organized as follows. 
Section~\ref{prelimilaries} introduces the preliminaries of goal-conditioned RL and related work.
In Section~\ref{method}, we first present the overview of MFRS, followed by specific implementations in detail and the final integrated algorithm.
Simulated experiments on several robotic manipulation tasks with the dynamic target and obstacles are conducted in Section~\ref{experiments}.
Section~\ref{sim2real} demonstrates the application of our method to a real robot, and Section~\ref{conclusion} presents our concluding remarks.

\section{Preliminaries and Related Work}\label{prelimilaries}
\subsection{Goal-Conditioned Reinforcement Learning}
We consider a discounted infinite-horizon goal-conditioned Markov decision process (MDP) framework, defined by a tuple $\langle\mathcal{S}, \mathcal{G}, \mathcal{A}, \mathcal{T}, R, \gamma\rangle$, where $\mathcal{S}$ is the set of states, $\mathcal{G}$ is the set of goals which is the subset of states $\mathcal{G}\subseteq\mathcal{S}$, $\mathcal{A}$ is the set of actions, $\mathcal{T}: \mathcal{S}\times \mathcal{A}\times \mathcal{S}\to [0,1]$ is the state transition probability, $R: \mathcal{S}\times \mathcal{A}\times \mathcal{G} \to\mathbb{R}$ is the goal-conditioned reward function, and $\gamma$ is a discount factor $\gamma\in (0,1]$.

We aim to learn a goal-conditioned policy that can achieve multiple goals simultaneously. Considering a goal-reaching agent interacting with the environment, in which every episode starts with sampling an initial state $s_0\sim \rho_0$ and a goal $g\sim \rho_g$, where $\rho_0$ and $\rho_g$ denote the distribution of the initial state and goal, respectively. At every timestep $t$ of the episode, the agent selects an action according to a goal-conditioned policy $a_t\sim \pi(\cdot|s_t,g)$. Then it receives an instant reward $r_t = r(s_t, a_t, g)$ that indicates whether the agent has reached the goal, and the next state is sampled from the distribution of state transition probability $s_{t+1}\sim p(\cdot|s_t, a_t)$.

Formally, the objective of goal-conditioned RL is to find an optimal policy $\pi^{*}$ that maximizes the expected return, which is defined as the discounted cumulative rewards:
\begin{equation}
    J(\pi) = \mathbb{E}_{g\sim \rho_g, \tau\sim d^{\pi}(\cdot|g)} \left[\sum_{t=0}^{\infty}\gamma^{t}r_{t} \right],
\end{equation}
under the distribution
\begin{equation}
    d^{\pi}(\tau|g) = \rho_0(s_0)\prod_{t=0}^{\infty}\pi(a_t|s_t, g)p(s_{t+1}|s_t,a_t),
\end{equation}
where $\tau = (s_0, a_0, s_1, a_1, ...)$ is the learning episode. Based on UVFA~\citep{schaul2015universal}, goal-conditioned RL algorithms rely on the appropriate estimation of goal-conditioned action-value function $Q^{\pi}(s,a,g)$, which describes the expected return when performing action $a$ in state $s$, goal $g$ and following $\pi$ after:
\begin{equation}
Q^{\pi}(s_t, a_t, g) = \mathbb{E}_{(r_{i\geq t}, s_{i>t})\sim E, (a_{i>t})\sim\pi, g\sim\rho_{g}}\left[\sum_{i=t}^{\infty} \gamma^{i-t}r_i\right],
\end{equation}
where $E$ stands for the environment.


In this work, we follow the standard off-policy actor-critic framework~\citep{silver2014deterministic, heess2015learning, mnih2016asynchronous, fujimoto2018addressing, haarnoja2018soft}, where a replay buffer $D$ is used to store the transition tuples $(s_{t}, a_{t}, s_{t+1}, r_t, g)$ as the experience for training.
The policy is referred as the actor and the action-value function as the critic, which can be represented as parameterized approximations using deep neural network (DNN) $\pi_{\bm{\theta}}(s,g)$ and $Q_{\bm{\phi}}(s,a,g)$, where $\bm{\theta}$ and $\bm{\phi}$ denote the parameters of the actor-network and critic-network, respectively. 

Actor-critic algorithms maximize the expected return by alternating between policy evaluation and policy improvement. 
\textcolor{black}{In the policy evaluation phase, the critic estimates the action-value function of the current policy $\pi_{\bm{\theta}}$, and the objective for the critic is minimizing the square of the Bellman error \citep{bellman1966dynamic}:
\begin{equation}
L(\bm{\phi}) = \mathbb{E}_{(s_{t},a_{t},s_{t+1},r_{t},g)\sim D} \left[y_{t} - Q_{\bm{\phi}}(s_{t},a_{t},g)\right]^2,
\end{equation}
where $y_{t} = r_{t} + \gamma\mathbb{E}_{a_{t+1}\sim\pi_{\bm{\theta}}(\cdot|s_t,g)}\left[Q_{\bm{\phi}}(s_{t+1},a_{t+1},g)\right]$.
In the policy improvement phase, the actor is optimized by maximizing the expected action-value function, and thus the objective function can be written as
\begin{equation}
J(\bm{\theta}) = \mathbb{E}_{(s_t,g)\sim D, a_t\sim\pi_{\bm{\theta}}(\cdot|s_t,g)} \left[Q_{\bm{\phi}}(s_t,a_t,g)\right].
\end{equation}
Hereafter, a gradient descent step $\bm{\phi}\leftarrow \bm{\phi} - \beta\nabla_{\bm{\phi}}L(\bm{\phi})$ and a gradient ascent step $\bm{\theta}\leftarrow \bm{\theta} + \alpha\nabla_{\bm{\theta}}J(\bm{\theta})$ can be taken to update the parameters of the critic network and actor network, respectively, where $\beta$ and $\alpha$ denote the learning rates.
The minimizing and maximizing strategies continue until $\bm{\theta}$ and $\bm{\phi}$ converge~\citep{williams1992simple,duan2016benchmarking}.}


\subsection{Related Work}
Goal-conditioned RL aims to learn a universal policy that can master reaching multiple goals simultaneously~\citep{liu2022goal}.
Following UVFA~\citep{schaul2015universal}, the reward function in goal-conditioned RL is typically a simple unshaped sparse reward, indicating whether the agent has reached the goal. Therefore, the agent can gain nothing from the environment until eventually getting into the goal region, which usually requires a large number of sampled episodes and brings about the problem of sample inefficiency, especially in complex environments such as robotic manipulation tasks.

A straightforward way to alleviate the sample efficiency problem is to replace the reward with a distance measure between the agent's current state and goal.
\citep{wu2018laplacian, ghosh2018learning, warde2018unsupervised, hartikainen2019dynamical, durugkar2021adversarial} propose different distance metrics to provide the agent with an accurate and informative reward signal indicating the distance to the target. 
However, they cannot guarantee that the optimal learned policy with the shaped reward will be the same as the original reward, which yields an additional local optima problem~\citep{trott2019keeping}.
In contrast, our method holds the property of optimal policy invariance and preserves a fully-informative dense reward simultaneously.

For the local optima problem generated by directly replacing the reward function, reward shaping is a popular way of incorporating domain knowledge into policy learning without changing the optimal policy.
\citet{ng1999policy} propose the potential-based reward shaping (PBRS), which defines the strict form of shaping reward function with proof of sufficiency and necessity that the optimal policy remains unchanged.
\citet{wiewiora2003principled} extend the input of potential function to state-action pair, allowing the incorporation of behavioral knowledge.
\citet{devlin2012dynamic} introduce a time parameter to the potential function, allowing the generalization to dynamic potentials.
Further, \citet{harutyunyan2015expressing} combine the above two extensions and propose the dynamic potential-based advice (DPBA) method, which can translate arbitrary reward function into the form of potential-based shaping reward by learning a secondary potential function together with the policy. 
All these methods hold the theoretical property of optimal policy invariance. 
However, they typically calculate the shaping reward based on simple distance metrics in goal-conditioned RL, which may fail to provide sufficient information in a highly complex dynamic environment.
In contrast, we establish our shaping reward function according to the magnetic field intensity with nonlinear and anisotropic distribution, carrying more informative and conducive information about the optimization landscape.

\textcolor{black}{
For the local optima problem, apart from the reward shaping methods, Learning from Demonstration (LfD) is also considered as an effective solution that introduces a human-in-the-loop learning paradigm with expert guidance.
\citep{wu2022prioritized} proposes a human guidance-based PER mechanism to improve the efficiency and performance of RL algorithm, and an incremental learning-based behavior model that imitates human demonstration to relieve the workload of human participants.
\citep{wu2022toward} develops a real-time human-guidance-based DRL method for autonomous driving, in which an improved actor-critic architecture with modified policy and value network is introduced.
On the other hand, the local optima problem also exists when performing gradient descent to solve the ideal parameters in RL, especially in the field of control systems.
To obtain the global optimal solution, \citep{bai2019adaptive} exploits the current and some of the past gradients to update the weight vectors in neural networks, and proposes a multi-gradient recursive (MGR) RL scheme, which can eliminate the local optima problem and guarantee a faster convergence than gradient descent methods.
\citep{bai2021event} extends MGR to the distributed RL to deal with the tracking control problem of uncertain nonlinear multi-agent systems, which further decreases the dependence on network initialization.}

Another line of research that tackles the sample efficiency problem due to sparse reward in goal-conditioned RL is the hindsight relabeling strategy, which can be traced back to the famous hindsight experience replay (HER)~\citep{andrychowicz2017hindsight}.
HER makes it possible to reuse unsuccessful trajectories in the replay buffer by relabeling the ``desired goals'' with certain ``achieved goals'', significantly reducing the number of sampled transitions required for learning to complete the tasks.
Recently, researchers have been interested in designing advanced relabeling methods for goal sampling to improve the performance of HER~\citep{ding2019goal, zhang2020automatic, mccarthy2021imaginary, fang2019curriculum, pitis2020maximum, kuang2020goal}.
However, HER and its variants only focus on achieving the desired goal but fail to consider the obstacles. In contrast, our method can handle tasks with multiple dynamic obstacles and learn to reach the desired goal while avoiding the obstacles with high effectiveness and efficiency.

\section{Magnetic Field-based Reward Shaping}~\label{method}
In this section, we first introduce the overview of MFRS that considers the target and obstacles as permanent magnets and exploits the physical property of the magnetic field to build the shaping reward function for goal-conditioned tasks.
Then, we illustrate the calculation of the resulting magnetic field intensity for different shapes of permanent magnets in detail
and use several normalization techniques to unify the intensity distribution and prevent the value exploding problem.
Finally, we define the form of our magnetic reward function and transform it into a potential-based one by learning a secondary potential function concurrently to ensure the optimal policy invariance property.

\subsection{Method Overview}\label{overview}
For the sample inefficiency problem due to the sparse reward in goal-conditioned RL, reward shaping is a practical approach to incorporating domain knowledge into the learning process without changing the original optimal policy.
However, existing reward shaping methods typically use distance measures with a linear and isotropic distribution to build the shaping reward, which may fail to provide sufficient information about the complex environments with the dynamic target and obstacles.

Therefore, it is necessary to develop a more sophisticated reward shaping method for this specific kind of task.
For a goal-conditioned RL task with the ever-changing target and obstacles, our focus is to head toward the target while keeping away from the obstacles. 
Naturally, the agent is attracted by the target and repelled by the obstacles, which is analogous to the physical phenomenon that magnets of the same polarity repel each other while those of different polarities attract each other.
Motivated by this, we model the dynamic target and obstacles as permanent magnets and establish our shaping reward according to the intensity of the magnetic field generated by these permanent magnets.
The magnetic field with a nonlinear and anisotropic distribution can provide more accessible and conducive information about the optimization landscape, thus introducing a more sophisticated shaping reward for goal-conditioned RL tasks with dynamic target and obstacles.
Fig.~\ref{fig:distr-camp} visualizes an example comparison of the shaping reward distribution between the traditional distance-based setting and our magnetic field-based setting.

\begin{figure}[tb]
\centering
\subfigure[Distance-based setting]{\includegraphics[width=0.20\textwidth]{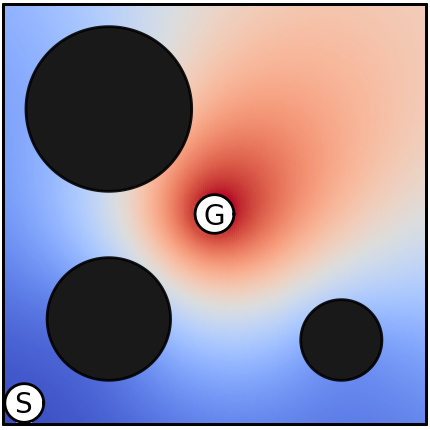}}\hspace{2em}
\subfigure[Magnetic field-based setting]{\includegraphics[width=0.20\textwidth]{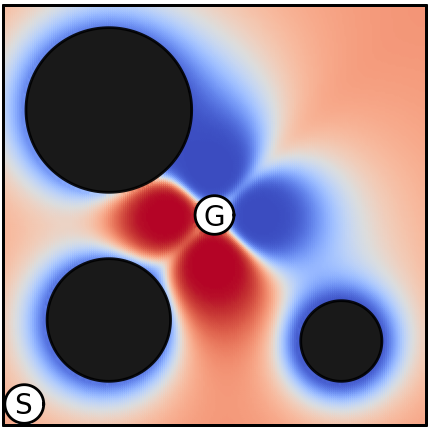}}
\caption{An example comparison of the shaping reward distribution between the traditional distance-based setting and our magnetic field-based setting. 
$S$ and $G$ are the start point and goal. Black circles indicate the obstacles in different sizes.
\textcolor{black}{Since the shaping reward is positively related to the magnetic field intensity, areas painted with warmer color possess a larger shaping reward.}
}
\label{fig:distr-camp}
\end{figure}

According to the physical property of permanent magnets, the agent will sense a higher intensity value in the magnetic field when getting closer to the target or away from the obstacles.
The shaping reward varies more intensively in the near region of the target and obstacles.
The target area uses a far larger shaping reward to facilitate its attraction to the agent, while the area around obstacles will receive a far smaller shaping reward as a ``danger zone'' to keep the agent at a distance.
The nonlinearity of the magnetic field reinforces the process that the agent is attracted by the target and repelled by the obstacles.
Being different from the distance-based setting, the derivative of the magnetic field intensity value with respect to the agent's position will also increase when approaching the target or moving away from the obstacles.
This property makes the shaping reward vary more intensively when the agent explores the environments, especially in the near regions around the target and obstacles, as illustrated in Fig.~\ref{fig:distr-camp}-(b).
Under this setting, the agent can be informed of an extra variation trend of the reward function together with the reward value itself. Thus it knows ``how'' to approach the target in addition to ``where'' to go.
Therefore, the nonlinear characteristic of the magnetic field can provide more useful position information for the algorithm to improve its performance.

For a goal-conditioned RL task, the optimal policy should be the steepest path from the initial state to the target while avoiding the obstacles.
Considering the points with the same Euclidean distance to the target, the distance-based reward shaping methods assign the same shaping reward (or potential) value to these points according to the isotropic distribution of the distance function.
However, the points between the initial state and the target should be preferred than the others since they can eventually lead to a shorter path to the target.
It is intuitive and rational to assign higher values to those points, which can more explicitly embed the orientation information of the target into the shaping reward.
Our magnetic field-based approach can realize this consideration by adjusting the target magnet's magnetization direction towards the agent's initial state.
As shown in Fig.~\ref{fig:distr-camp}-(b), the shaping reward distribution of the target area is anisotropic, where the part between the initial state and the goal exhibits a larger value than the other parts.
Accordingly, the agent will have a clear insight of the heading direction towards the target.
In brief, the anisotropic feature of the magnetic field can provide more orientation information about the goal-conditioned RL tasks.

\textcolor{black}{Based on the above insights, we propose a magnetic field-based reward shaping (MFRS) method for goal-conditioned RL tasks, which can provide abundant and conducive information on the position and orientation of the target and obstacle due to the nonlinear and anisotropic distribution of the magnetic field.
To this end, an additional dense shaping reward according to the magnetic field intensity is added to the original sparse reward with the guarantee of optimal policy invariance.
As a result, the agent can enjoy an informative reward signal indicating where and how far the target and obstacles are in each timestep, and maintain the same optimal policy under the sparse reward condition.
Therefore, the sample inefficiency problem caused by the sparse reward can be alleviated.}
We consider the target and obstacles as permanent magnets, and calculate the intensity values of the resulting magnetic field using physical laws.

\subsection{Calculation of Magnetic Field Intensity}\label{cal-mfi}
\textcolor{black}{
For permanent magnets in regular shapes, we can directly calculate the intensity function analytically, using different formulas for different shapes.
We take the spherical and cuboid permanent magnets as examples to calculate the distribution of magnetic field intensity in three-dimensional (3-D) space for illustration.
As for the permanent magnets in irregular shapes, the intensity distribution of magnetic field can be obtained by analog platforms using physics simulation software such as COMSOL~\citep{pryor2009multiphysics}.
}

\begin{figure}[tb]
\centering
\subfigure[Spherical magnet]{\includegraphics[width=0.24\textwidth]{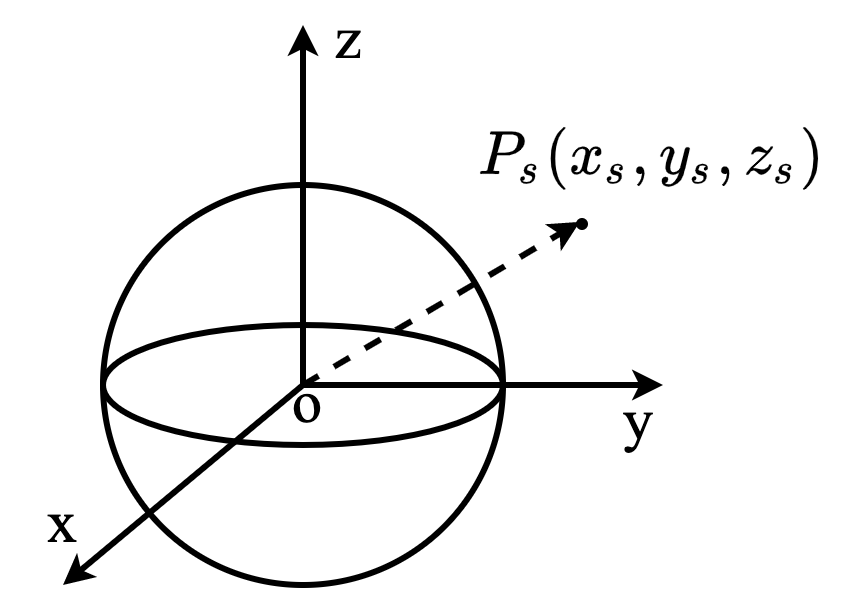}}
\subfigure[Cuboid magnet]{\includegraphics[width=0.24\textwidth]{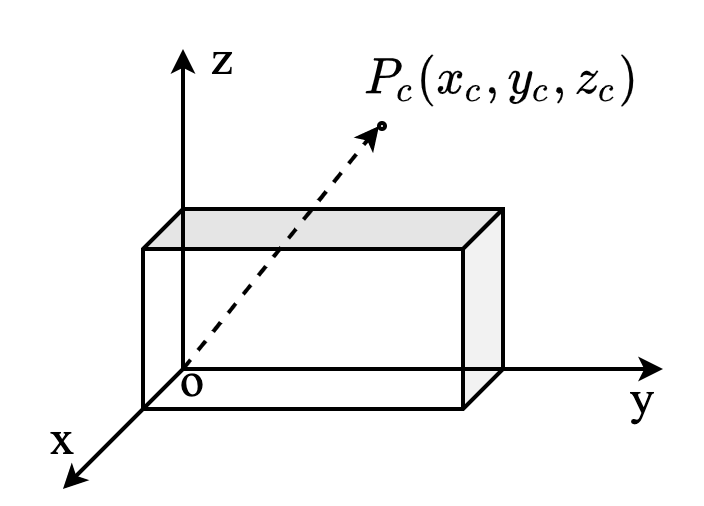}}
\caption{Magnetic field coordinate systems of the spherical and cuboid magnets, in which $P_{s}(x_s,y_s,z_s)$ and $P_{c}(x_c,y_c,z_c)$ are arbitrary given points.}
\label{fig:mag-coord}
\end{figure}

We assume the spherical and cuboid magnets are both saturatedly magnetized along the positive direction of the $z$-axis, and Fig.~\ref{fig:mag-coord} presents the magnetic field coordinate systems of the spherical and cuboid magnets, respectively, for intensity calculation.
To obtain the magnetic field intensity at an arbitrary given point in the field of spherical and cuboid magnets, we first need to transfer the point from the environment coordinate system to the magnetic field coordinate system of the spherical and cuboid magnets.
~\citet{deakin19983} has introduced the form of 3-D conformal coordinate transformation that combines axes rotation, scale change, and origin shifts. 
Analogously, our coordinate system transformation of points in space is a Euclidean transformation that only consists of rotation and translation, which can be defined as
\begin{equation}
    \begin{bmatrix}
        p^{m}_{x} \\ p^{m}_{y} \\ p^{m}_{z}
    \end{bmatrix}
    = R_{\theta_x} R_{\theta_y} R_{\theta_z} 
    \begin{bmatrix}
        p^{e}_{x} \\ p^{e}_{y} \\ p^{e}_{z}
    \end{bmatrix} + 
    \begin{bmatrix}
        T_{x} \\ T_{y} \\ T_{z}
    \end{bmatrix},
\label{pos-trans}
\end{equation}
where $p^{m}_{x}, p^{m}_{y}, p^{m}_{z}$ and $p^{e}_{x}, p^{e}_{y}, p^{e}_{z}$ are 3-D coordinates of the point in the magnetic field coordinate system and environment coordinate system, respectively. 
$T_{x}$, $T_{y}$, and $T_{z}$ are translations between the origins of the two coordinate systems along the $x$-axis, $y$-axis, and $z$-axis. 
$R_{\theta_x}$, $R_{\theta_y}$, and $R_{\theta_z}$ are the rotation matrices generated by the rotation angle of the magnetic field coordinate system with respect to the environment coordinate system around the $x$-axis, $y$-axis, and $z$-axis, which can be expressed by 
\begin{equation}
    \left\{\begin{array}{lr}
    R_{\theta_{x}} = \begin{bmatrix}
    1 & 0 & 0 \\
    0 & \cos\theta_{x} & -\sin\theta_{x} \\
    0 & \sin\theta_{x} & \cos\theta_{x} \end{bmatrix}
    \\
    R_{\theta_{y}} = \begin{bmatrix}
    \cos\theta_{y} & 0 & \sin\theta_{y} \\
    0 & 1 & 0 \\
    -\sin\theta_{y} & 0 & \cos\theta_{y} \end{bmatrix}
    \\
    R_{\theta_{z}} = \begin{bmatrix}
    \cos\theta_{z} & -\sin\theta_{z} & 0 \\
    \sin\theta_{z} & \cos\theta_{z} & 0 \\
    0 & 0 & 1 \end{bmatrix},
    \end{array}\right.
\label{rot-mat}
\end{equation}
where the positive sense of rotation angle is determined by the right-hand screw rule.

Now we have implemented the 3-D conformal coordinate transformation. Then we will illustrate how to calculate the magnetic field intensity of the points in the magnetic field coordinate system. For the spherical permanent magnet, given an arbitrary point $(x_s, y_s, z_s)$ in its magnetic field coordinate system, we first need to convert the point from Cartesian coordinates to Spherical coordinates by the formulae we defined below:
\begin{equation}
    \left\{\begin{array}{lr}
    r_s = \sqrt{x_s^2 + y_s^2 + z_s^2} \\
    \theta_s = \arccos \frac{z_s}{r_s + \epsilon} \\
    \phi_s = \arctan \frac{y_s}{x_s + \epsilon},
    \end{array}\right.
\end{equation}
where $r_s$, $\theta_s$, and $\phi_s$ denote the radius, inclination, and azimuth in Spherical coordinates, respectively, and $\epsilon$ is a small constant. 
Let $a_s$ denote the radius of the spherical magnet, then the magnetic field intensity of the spherical magnet at given point $(r_s, \theta_s, \phi_s)$ can be calculated analytically as
\begin{equation}
    \left\{\begin{array}{lr}
    \begin{aligned}
    H_{x}^{s} = &\frac{M_s}{4\pi} \int_{0}^{\pi} \int_{0}^{2\pi}
    [a_s^3\sin^3\theta_0\cos\varphi_0(r_s\cos\theta_s - a_s\cos\theta_0)] \\ 
    &/ [(r_s^2 + a_s^2 - 2a_s r_s(\cos\theta_s\cos\theta_0 + \\ &\sin\theta_s\sin\theta_0\cos(\varphi_s - \varphi_0)))^{3/2} + \epsilon] d\varphi_0 d\theta_0
    \end{aligned}
    \\
    \begin{aligned}
    H_{y}^{s} = &\frac{M_s}{4\pi} \int_{0}^{\pi} \int_{0}^{2\pi}
    [a_s^3\sin^3\theta_0\sin\varphi_0(r_s\cos\theta_s - a_s\cos\theta_0)] \\ 
    &/ [(r_s^2 + a_s^2 - 2a_s r_s(\cos\theta_s\cos\theta_0 + \\ &\sin\theta_s\sin\theta_0\cos(\varphi_s - \varphi_0)))^{3/2} + \epsilon] d\varphi_0 d\theta_0
    \end{aligned}
    \\
    \begin{aligned}
    H_{z}^{s} = &\frac{M_s}{4\pi} \int_{0}^{\pi} \int_{0}^{2\pi}
    [a_s^3\sin^3\theta_0(a_s\sin\theta_0 - r_s\cos\theta_s \\ 
    &\cos(\varphi_s - \varphi_0))] / [(r_s^2 + a_s^2 - 2a_s r_s(\cos\theta_s\cos\theta_0 + \\ &\sin\theta_s\sin\theta_0\cos(\varphi_s - \varphi_0)))^{3/2} + \epsilon] d\varphi_0 d\theta_0
    \end{aligned}
    \\
    H_{s} = \sqrt{(H_{x}^{s})^2 + (H_{y}^{s})^2 + (H_{z}^{s})^2},
    \end{array}\right.
\label{mag-sphere}
\end{equation}
where $M_s$ is the magnetization intensity of the spherical magnet. 
$H_x^s$, $H_y^s$, and $H_z^s$ are the calculated intensities of the magnetic field along the $x$-axis, $y$-axis, and $z$-axis for a spherical magnet, and $H_s$ is the required value of magnetic field intensity for the given point $(x_s, y_s, z_s)$ in the field of the spherical magnet.

For the cuboid magnet, let $l_c, w_c, h_c$ denote the length, width, and height of the cuboid magnet along the $x$-, $y$-, and $z$-axis, respectively. Then the magnetic field intensity at a given point $(x_c, y_c, z_c)$ can be calculated as
\begin{equation}
\label{mag-cuboid}
    \left\{\begin{array}{lr}
    \begin{aligned}
    H_{x}^{c} = &-\frac{M_c}{8\pi}[\Gamma(l_c - x_c, w_c - y_c, z_c)+\Gamma(l_c - x_c, y_c, z_c) \\
       &-\Gamma(x_c, w_c - y_c, z_c)-\Gamma(x_c, y_c, z_c)]
    \end{aligned}
    \\
    \begin{aligned}
    H_{y}^{c} = &-\frac{M_c}{8\pi}[\Gamma(w_c - y, l_c - x_c, z_c)+\Gamma(w_c - y_c, x_c, z_c) \\            &-\Gamma(y_c, l_c - x_c, z_c)-\Gamma(y_c, x_c, z_c)]
    \end{aligned}
    \\
    \begin{aligned}
    H_{z}^{c} = &-\frac{M_c}{8\pi}[\Psi(w_c - y_c, l_c - x_c, z_c) + \Psi(y_c, l_c - x_c, z_c) \\                &+\Psi(l_c - x_c, w_c - y_c, z_c) + \Psi(x_c, w_c - y_c, z_c) \\ 
                  &+\Psi(w_c - y_c, x_c, z_c) + \Psi(y_c, x_c, z_c) \\ &+\Psi(l_c - x_c, y_c, z_c)+\Psi(x_c, y_c, z_c)]
    \end{aligned}
    \\
    H_{c} = \sqrt{(H_{x}^{c})^2+(H_{y}^{c})^2+(H_{z}^{c})^2},
    \end{array}\right.
\end{equation}
where $M_c$ denotes the magnetization intensities of the cuboid magnet. $H_x^c$, $H_y^c$, and $H_z^c$ are the calculated intensities of the magnetic field along the $x$-axis, $y$-axis, and $z$-axis for a cuboid magnet, and $H_c$ is the required value of the magnetic field intensity for the given point $(x_c, y_c, z_c)$ in the field of the cuboid magnet.
$\Gamma(\gamma_{1},\gamma_{2},\gamma_{3})$ and $\Psi(\varphi_{1},\varphi_{2},\varphi_{3})$ are two auxiliary functions to simplify calculations as defined by
\begin{equation}
    \left\{\begin{array}{lr}
    \Gamma(\gamma_{1},\gamma_{2},\gamma_{3}) \!\! = \!\! \left.\ln\!\left(\frac{\sqrt{(\gamma_{1})^2+(\gamma_{2})^2+(\gamma_{3}-z_{0})^2}-\gamma_{2}}{{\sqrt{(\gamma_{1})^2+(\gamma_{2})^2+(\gamma_{3}-z_{0})^2}+\gamma_{2}}+\epsilon} +\epsilon\right)\!\right|_{z_{0}=0}^{z_{0}=h_c} \\
    
    \Psi(\varphi_{1},\varphi_{2},\varphi_{3}) \!\! = \!\! \left.\arctan\frac{\varphi_{1}(\varphi_{3}-z_{0})}{\varphi_{2}\sqrt{(\varphi_{1})^2+(\varphi_{2})^2+(\varphi_{3}-z_{0})^2}+\epsilon}\right|_{z_{0}=0}^{z_{0}=h_c}.
    \end{array}\right.
\label{mag-cub}
\end{equation}

With these analytical expressions, we are now able to obtain the specific value of magnetic field intensity at any given point in the field of spherical magnet and cuboid magnet.
For simplicity, we let $M_s = M_c = 4\pi$ as the value of magnetization intensity only determines the magnitude of magnetic field intensity, rather than its distribution.
Fig.~\ref{fig:mag-distr} visualizes the distributions of magnetic field intensity generated by the spherical and cuboid magnets in 3-D space, where the sphere's radius is $0.02$, the length, width, and height of the cuboid are $0.1$, $0.4$, $0.05$, respectively.
Areas with warmer colors have a higher intensity value in the magnetic field.
It can be observed that the surfaces of the sphere and cuboid magnets exhibit the highest intensity value, and the intensity decreases intensively as the distance to the magnet increases.

\textcolor{black}{
Towards the environments with higher dimensions (e.g., more than three dimensions), though we cannot directly calculate the magnetic field intensity since the target/obstacles are not physical entities and thus can not be considered as permanent magnets in the physical sense, the concept and properties of the magnetic field can still be extended to the scenario with a high-dimensional goal in a mathematical sense.
For instance, the magnetic field intensity of a spherical permanent magnet along the axis of magnetization (taking the $z$-axis as an example) can be calculated as below:
\begin{equation}
    H = \left\{\begin{array}{lr}
    \frac{2M_{0}r^{3}}{3{|z|^{3}}}\hat{z}, \qquad (|z|\geq r) \\
    -\frac{M_{0}}{3}\hat{z}, \qquad (|z|\textless r),
    \end{array}\right.
\end{equation}
where $M_{0}$ is the magnetization intensity, $r$ is the radius of the spherical magnet, and $z$ denotes the coordinate on the $z$-axis in the magnetic field coordinate system.
As we can see, the intensity value is inversely proportional to the third power of the distance to the origin of magnetic field coordinate system along the magnetization axis outside the magnet.
This property of the spherical magnets can be extended to high-dimensional goals by simply replacing $|z|$ with a distance metric (e.g., L2-norm) between the current state and the target state, which can still maintain a nonlinear and intensively-varying distribution of shaping rewards that provides sufficient information about the optimization landscape.}

\begin{figure}[tb]
\centering
\subfigure[Spherical magnet]{\includegraphics[width=0.24\textwidth]{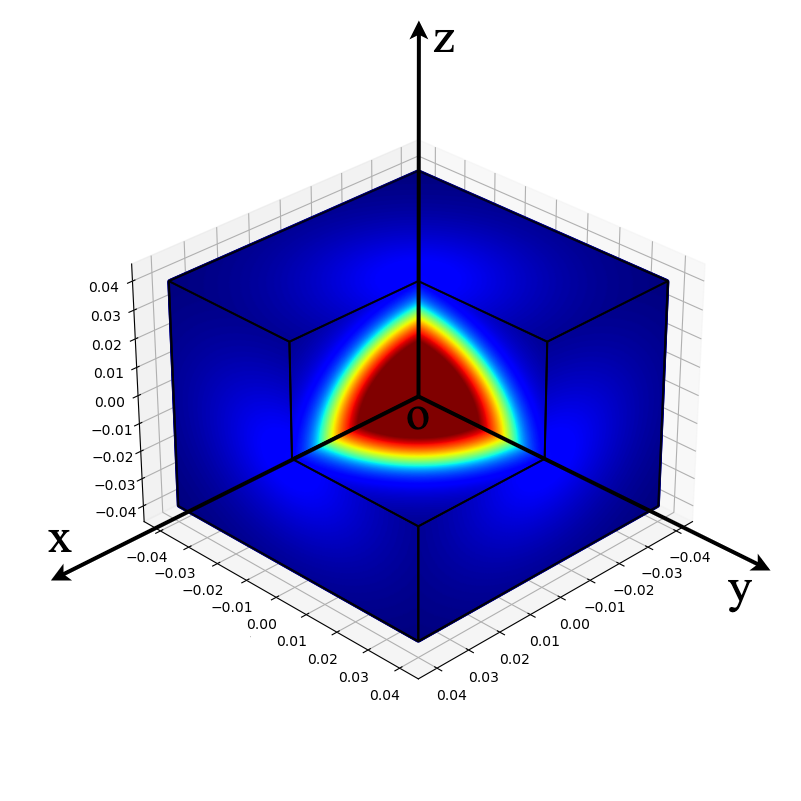}}
\subfigure[Cuboid magnet]{\includegraphics[width=0.24\textwidth]{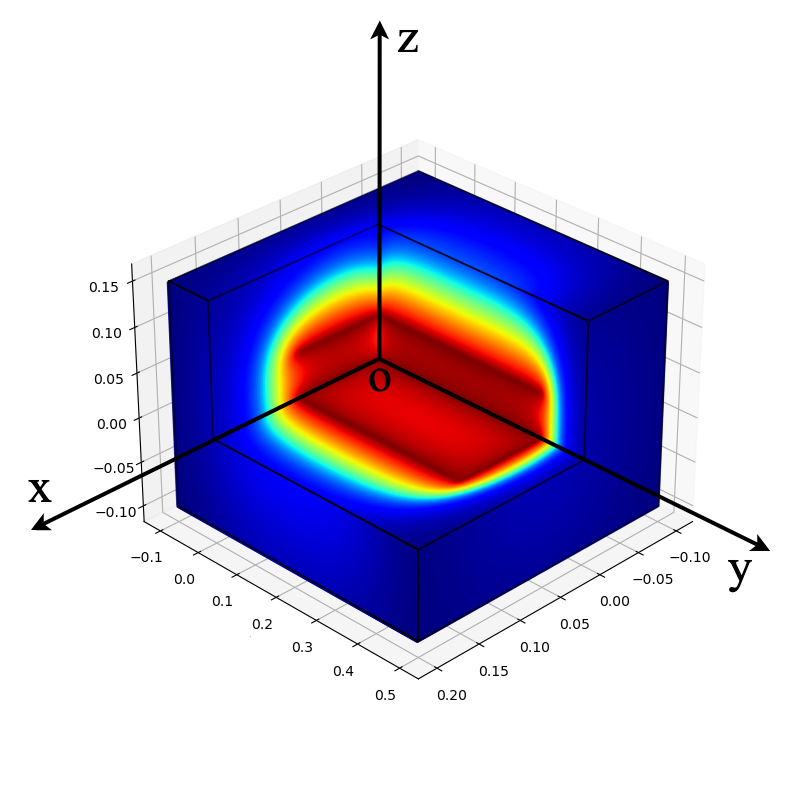}}
\caption{Distributions of the magnetic field intensity in 3-D space generated by the spherical and cuboid magnets. Areas painted with warmer color possess a larger intensity value of the magnetic field.}
\label{fig:mag-distr}
\end{figure}

\subsection{Magnetic Reward Function}\label{mag-srf}
Considering a goal-conditioned task with 1 target and $N$ obstacles. 
\textcolor{black}{
Let $\mathcal{M}_{T}(P_a^{T})$ and $\mathcal{M}_{O_i}(P_a^{O_i}) \ (i=1,2,...,N)$ denote the functions that calculate the intensity values $H_{T}$ and $H_{O_i}$ of the magnetic fields generated by the target magnet and obstacle magnets, respectively, where $P_a^T$ and $P_a^{O_i}$ are the agent's positions in the magnetic field coordinate system of target and obstacle magnets. 
Taking the spherical target and cuboid obstacle as an example, the intensity functions $\mathcal{M}_{T}$ and $\mathcal{M}_{O}$ can be defined as Eq.~\ref{mag-sphere} and Eq.~\ref{mag-cuboid} in Section~\ref{cal-mfi}, respectively.}
The intensity values cannot be directly combined since the scales of magnetic field intensity for various magnets are different using separate ways of calculation. 
It may amplify the effect of one or some magnets and give an incorrect reward signal to the agent if the intensity scales of these magnets are much larger than others. 
Therefore, we need to standardize the calculated intensity values $H_{T}$ and $H_{O_i}$.

Intuitively, we use the $z$-score method to unify the unevenly distributed intensity values into a standard Gaussian distribution for the target and obstacles magnets as
\begin{equation}
\left\{\begin{array}{lr}
    H_{T}^{\mathcal{N}} = \frac{H_{T}-\mu_{T}}{\sigma_{T} + \epsilon} \\
    H_{O_i}^{\mathcal{N}} = \frac{H_{O_i}-\mu_{O_i}}{\sigma_{O_i} + \epsilon} \quad (i=1,2,...,N),
\end{array}\right.
\label{std-h}
\end{equation}
where $\epsilon$ is a small constant, $\mu_{T}$ and $\mu_{O_i}$ denote the means of magnetic field intensity for the target and obstacles magnet, and $\sigma_{T}$ and $\sigma_{O_i}$ denote the standard deviation. 
Unfortunately, the actual values of the mean and standard deviation cannot be obtained directly since the agent knows nothing about the specific intensity distribution at the beginning of training.
Hence, we turn to estimate these values by introducing an extra magnet buffer $\mathcal{D}_\mathcal{M}$ that stores the calculated intensity values $H_{T}$ and $H_{O_i}$.
During the learning process, we concurrently update the values of the mean and standard deviation for $H_{T}$ and $H_{O_i}$ using the up-to-date buffer $\mathcal{D}_\mathcal{M}$, which are used to standardize the values of calculated magnetic field intensity in the next learning episode.
According to the law of large numbers, as the number of collected data in $\mathcal{D}_\mathcal{M}$ increases, the estimated values of the mean and standard deviation will converge to their actual values.

After implementing the standardization of the magnet field intensity, the unified intensity of the target and obstacles magnets can be combined to a value that represents the overall intensity of the magnetic field generated by the target and obstacles together, which can be expressed as
\begin{equation}
    H_{com} = H_{T}^{\mathcal{N}} - \frac{1}{N}\sum_{i=1}^{N} H_{O_i}^{\mathcal{N}} \quad (i=1,2,...,N).
\label{H_com}
\end{equation}
Recall our motivation in \ref{overview}, we want the agent to be attracted by the target and repelled by the obstacles when exploring the environment.
In addition, we consider that the target's attraction has the exact extent of the effect as all obstacles' repulsion.
Accordingly, we take the mean of unified intensity values among all the obstacles in the calculation and give a positive value to $H_{T}^{\mathcal{N}}$ while a negative value to the mean of $H_{O_i}^{\mathcal{N}}$ to define the combined intensity value $H_{com}$.

Due to the physical property of magnets, the intensity values will be tremendous in the near region of the target, as shown in Fig.~\ref{fig:distr-camp}.
Therefore, regarding $H_{com}$ as the magnetic reward could probably lead to a phenomenon that the value of the magnetic reward is orders of magnitude larger than the original reward.
To address the value exploding problem, we employ the Softsign function to normalize the combined intensity value within a bounded range of $[-1, 1]$ and use the output of the Softsign function as the magnetic reward $R^{m}$:
\begin{equation}
    R^{m} = \text{Softsign}(H_{com}) = \frac{H_{com}}{1 + \mid H_{com} \mid}.
\label{shaping_r}
\end{equation}

\begin{algorithm}[tb]
\caption{Magnetic reward function}
\label{srf-mfrs}

\KwIn{state $s$; goal $g$; \newline
intensity calculation functions $\mathcal{M}_T, \mathcal{M}_{O_i}$; \newline
$\mu_T, \sigma_T, \mu_{O_i}, \sigma_{O_i}\ (i=1,2,...,N)$}
\KwOut{magnetic field intensity $H_T, H_{O_i}$; \newline 
magnetic reward $R^{m}$}

Obtain agent's position $P_a$ according to $s$ \\
Obtain target's position $P_T$ and obstacles' positions $P_{O_i}$ according to $g$ \\

Transfer $P_a$ to $P_a^T, P_a^{O_i}$ using~(\ref{pos-trans}) according to $P_T, P_{O_i}$\\
$H_{T}\leftarrow \mathcal{M}_{T}(P_{a}^{T}), H_{O_i}\leftarrow \mathcal{M}_{O_i}(P_{a}^{O_i})$ \\
Calculate $H_{T}^{\mathcal{N}}$ and $H_{O_i}^{\mathcal{N}}$ using~(\ref{std-h}) according to $\mu_T, \sigma_T, \mu_{O_i}, \sigma_{O_i}$ \\
Calculate $H_{com}$ using~(\ref{H_com}) \\
$R^{m} \leftarrow \text{Softsign}(H_{com})$
\end{algorithm}

Algorithm~\ref{srf-mfrs} shows the calculation process of the magnetic reward function in our method.
Since we consider the obstacles as a part of the goal in goal-conditioned RL together with the target, we define the goal in our setting as $g := [P_T,P_{O_1},P_{O_2},...,P_{O_N}]$, where $P_T$ is the target's position and $P_{O_1}, P_{O_2},...,P_{O_N}$ are obstacles' positions.

\subsection{Magnetic Field-based Reward Shaping}
After calculating the magnetic reward $R^m$ that represents the overall intensity generated by the target and obstacles magnets with normalization techniques, we need to guarantee the optimal policy invariance of our method when using $R^m$ as the shaping reward.
PBRS~\cite{ng1999policy} defines the shaping reward function as $F(s,a,s') = \gamma\Phi(s') - \Phi(s)$ with proof of sufficiency and necessity that the optimal policy remains unchanged, where $\Phi(\cdot)$ denotes the potential function carrying human domain knowledge, and $\gamma\in[0,1]$ is the discount factor.
In this paper, we employ the DPBA~\citep{harutyunyan2015expressing} method to transform our magnetic reward $R^{m}$ into the form of potential-based reward shaping in the goal-conditioned RL setting.
To be specific, we need to achieve $F \approx R^{m}$,
where the potential function in $F$ can be learned in an on-policy way~\cite{suay2016learning} with a technique analogous to SARSA~\cite{rummery1994line}:
\begin{equation}
    \Phi_{t+1}(s,a,g) \leftarrow \Phi_{t}(s,a,g) + \eta\delta_t^{\Phi},
\label{phi-update}
\end{equation}
where $\eta$ is the learning rate of this secondary goal-conditioned potential function $\Phi$, and $\delta_t^{\Phi}$ denotes the temporal difference (TD) error of the state transition:
\begin{equation}
    \delta_t^{\Phi} = r_{t}^{\Phi} + \gamma\Phi_t(s_{t+1},a_{t+1},g) - \Phi_t(s_t,a_t,g),
\end{equation}
where $a_{t+1}$ is chosen using the current policy.
According to DPBA~\citep{harutyunyan2015expressing}, the value of $R^{\Phi}$ equals the negation of the expert-provided reward function, which is the magnetic reward in our method: $R^{\Phi} = -R^{m}$, and thus $r_{t}^{\Phi} = - r^{m}_{t}$.

Since we focus on the tasks with continuous state and action space, the goal-conditioned potential function can be parameterized using a deep neural network $\Phi_{\bm{\psi}}(s,a,g)$ with the weights $\bm{\psi}$.
Akin to the temporal-difference learning, we use the Bellman error as the loss function of this potential network as
\begin{equation}
    \mathcal{L}_{\Phi} = \frac{1}{2} \left((-r^{m}_{t} + \gamma\Phi_{\bm{\psi}}(s_{t+1},a_{t+1},g) - \Phi_{\bm{\psi}}(s_t,a_t,g)\right)^{2}.
\label{loss-func}
\end{equation}
Hereafter, the weights of $\Phi_{\bm{\psi}}(s,a,g)$ can be updated using the gradient descent as: 
\begin{equation}
    \bm{\psi}\leftarrow \bm{\psi} - \eta\nabla_{\bm{\psi}}\mathcal{L}_{\Phi}.
\label{update-psi}
\end{equation}
As the secondary goal-conditioned potential function $\Phi$ is updated every timestep, the potential-based shaping reward $F$ at each timestep can be expressed as
\begin{equation}
    f_{t} = \gamma\Phi_{t+1}(s_{t+1},a_{t+1},g) - \Phi_{t}(s_t,a_t,g).
\label{shaping-reward}
\end{equation}

\begin{algorithm}[tb]
\caption{MFRS for goal-conditioned RL}
\label{mfrs}

\KwIn{potential function $\Phi_{\bm{\psi}}$; replay buffer $\mathcal{D}_{\mathcal{R}}$; magnet buffer $\mathcal{D}_{\mathcal{M}}$; number of obstacles $N$; discount factor $\gamma$; off-policy RL algorithm $\mathbb{A}$} 
\KwOut{optimal goal-conditioned policy $\pi_{\bm{\theta}}^{*}$}

Initialize $\pi_{\bm{\theta}}$ arbitrarily, $\Phi_{\bm{\psi}}\leftarrow 0$ \\
Initialize replay buffer $\mathcal{D}_{\mathcal{R}}$ and magnet buffer $\mathcal{D}_{\mathcal{M}}$ \\
$\mu_T, \mu_{O_i} \leftarrow 0$ ; $\sigma_T, \sigma_{O_i} \leftarrow 1\ (i=1,2,...,N)$ \\

\For{episode = 1, E}{
    Sample an initial state $s_0$ and a goal $g$ \\
    \For{t = 0, H - 1}{
        $a_t\leftarrow \pi_{\bm{\theta}}(s_t,g)$ \\
        Execute $a_t$, observe $s_{t+1}$ and $r_t$ \\
        Calculate $H_{T}, H_{O_i}$ and $r_t^m$ using Algorithm~\ref{srf-mfrs} according to $s_{t+1}, g$ and $\mu_T, \sigma_T, \mu_{O_i}, \sigma_{O_i}$ \\
        Store $H_T, H_{O_i}$ in $\mathcal{D}_{\mathcal{M}}$ \\
        
        $a_{t+1}\leftarrow \pi_{\bm{\theta}}(s_{t+1},g)$ \\
        Update $\bm{\psi}$ using Eq.~(\ref{update-psi}) \\
        Calculate $f_t$ using Eq.~(\ref{shaping-reward}) \\
        $r_{t}'\leftarrow r_t + f_t$ \\
        Store transition $(s_t,a_t,s_{t+1},r_{t}',g)$ in $\mathcal{D}_{\mathcal{R}}$ \\
    }
    Update $\pi_{\bm{\theta}}$ using $\mathbb{A}$ \\
    $\mu_{T}\leftarrow mean(H_{T}), \sigma_{T}\leftarrow std(H_{T})$ \\
    $\mu_{O_i} \leftarrow mean(H_{O_i}), \sigma_{O_i} \leftarrow std(H_{O_i})$ \\
}
\end{algorithm}

Algorithm~\ref{mfrs} presents the integrated process of MFRS for goal-conditioned RL.
Next, we give a theoretical guarantee of optimal policy invariance in the goal-conditioned RL setting of our method in the Appendix, and a convergence analysis that the expectation of shaping reward $F$ will be equal to our magnetic reward $R^{m}$ when the goal-conditioned potential function $\Phi$ has converged in Theorem 1 below.

\begin{theorem}
    Let $\Phi$ be the goal-conditioned potential function updated by Eq.~(\ref{phi-update}) with the state transition matrix $\mathcal{T}$, where $R^{\Phi}$ equals the negation of our magnetic reward $R^{m}$. Then, the expectation of shaping reward $F$ expressed in Eq.~(\ref{shaping-reward}) will be equal to $R^{m}$ when $\Phi$ has converged.
\end{theorem}
\begin{proof}
    The goal-conditioned potential function $\Phi$ follows the update rule of the Bellman Equation~\citep{bellman1966dynamic}, which enjoys the same recursive relation when the potential value has converged to the TD-fixpoint $\Phi^{*}$. Thus, we have
    \begin{align}
         \Phi^{*}(s,a,g) =& R^{\Phi}(s,a,g) + \gamma\mathbb{E}[\Phi^{*}(s',a',g)]\nonumber\\
         =& -R^{m}(s,a,g) + \gamma\mathbb{E}[\Phi^{*}(s',a',g)].
    \end{align}

    
    According to Eq.~(\ref{shaping-reward}), we have the shaping reward $F$ with respect to the converged potential value $\Phi^{*}$:
    \begin{align}
         F(s,a,s',a',g) &= \gamma\Phi^{*}(s',a',g) - \Phi^{*}(s,a,g) \nonumber\\
         &= \gamma\Phi^{*}(s',a',g) + R^{m}(s,a,g) \nonumber\\
         &\quad - \gamma\mathbb{E}[\Phi^{*}(s',a',g)] \nonumber\\
         &= R^{m}(s,a,g) + \gamma(\Phi^{*}(s',a',g) \nonumber\\ 
         &\quad - \mathbb{E}[\Phi^{*}(s',a',g)]).
    \end{align}
    
    To obtain the expected shaping reward $\mathcal{F}(s,a,g)$, we take the expectation with respect to the state transition matrix $\mathcal{T}$:
    \begin{align}
        \mathcal{F}(s,a,g) &= \mathbb{E}[F(s,a,s',a',g)] \nonumber\\
        &= R^{m}(s,a,g) \! + \! \gamma\mathbb{E}\left[\Phi^{*}(s',a',g) \! - \! \mathbb{E}[\Phi^{*}(s',a',g)]\right] \nonumber\\
        &= R^{m}(s,a,g).
    \end{align}
    
    Hence, the expectation of shaping reward $F$ will be equal to our magnetic reward $R^{m}$ when the goal-conditioned potential function $\Phi$ converges to the TD-fixpoint.
\end{proof}

\section{Simulation Experiments}~\label{experiments}
To evaluate our method, we build a 3-D simulated robotic manipulation platform based on the software CoppeliaSim, formerly known as V-REP~\citep{rohmer2013v}, and design a set of challenging goal-conditioned RL tasks with continuous state-action space. 
Being analogous to the FetchReach-v1~\citep{plappert2018multi} task in OpenAI Gym~\citep{brockman2016openai}, our tasks require the agent to move the end-effector of a robotic arm to dynamic targets with different positions, but differ from FetchReach-v1 in additional dynamic obstacles that the agent should avoid.

\subsection{Environment Settings}\label{env-setting}
Sections \ref{compare-exp} and \ref{ablation} present the results and insightful analysis of our findings. In the experiments, we evaluate MFRS in comparison to several baseline methods as follows:
\begin{enumerate}
	\item \textbf{No Shaping (NS)}: As the basic sparse reward condition, it trains the policy with the original reward given by the environment without shaping.

	\item \textbf{PBRS}~\citep{ng1999policy}: It adds a calculated shaping reward $F=\gamma\Phi(s', g)-\Phi(s, g)$ to the original reward, where the potential function is built on Euclidean distance as $\Phi(s,g)=-d(p_a,p_g)+\frac{1}{N}\sum_{i=1}^{N}d(p_a,p_{o_i})$.
	
	\item \textbf{DPBA}~\citep{harutyunyan2015expressing}: It adds a learned shaping reward $F=\gamma\Phi_{t'}(s',a',g)-\Phi_{t}(s,a,g)\approx R^{\dagger}$ to the original reward, where the objective reward function is built on Euclidean distance as $R^{\dagger}=-d(p_a,p_g)+\frac{1}{N}\sum_{i=1}^{N}d(p_a,p_{o_i})$.
	
	\item \textbf{HER}~\citep{andrychowicz2017hindsight}: As the famous relabeling strategy in goal-conditioned RL, it relabels the desired goals in the replay buffer with the achieved goals in the same trajectories. In our experiments, we adopt the default ``final'' strategy in HER that chooses the additional goals corresponding to the final state of the environment.
	 
	\item \textcolor{black}{
	\textbf{Sibling Rivalry (SR)}~\citep{trott2019keeping}: It samples two sibling episodes for each goal simultaneously and uses each others' achieved goals as anti-goals $\overline{g}$. Then, it engineers an additional reward bonus $d(s, \overline{g})$ to encourage exploring and avoid local optima, where $d$ is the Euclidean distance.}
	
	\item \textcolor{black}{
	\textbf{Adversarial Intrinsic Motivation (AIM)}~\citep{durugkar2021adversarial}: It aims to learn a goal-conditioned policy whose state visitation distribution minimizes the Wasserstein distance to a target distribution for a given goal, and utilizes the Wasserstein distance to formulate a shaped reward function, resulting in a nonlinear reward shaping mechanism.}
\end{enumerate}

\begin{figure}[tb]
\centering
\includegraphics[width=0.56\columnwidth]{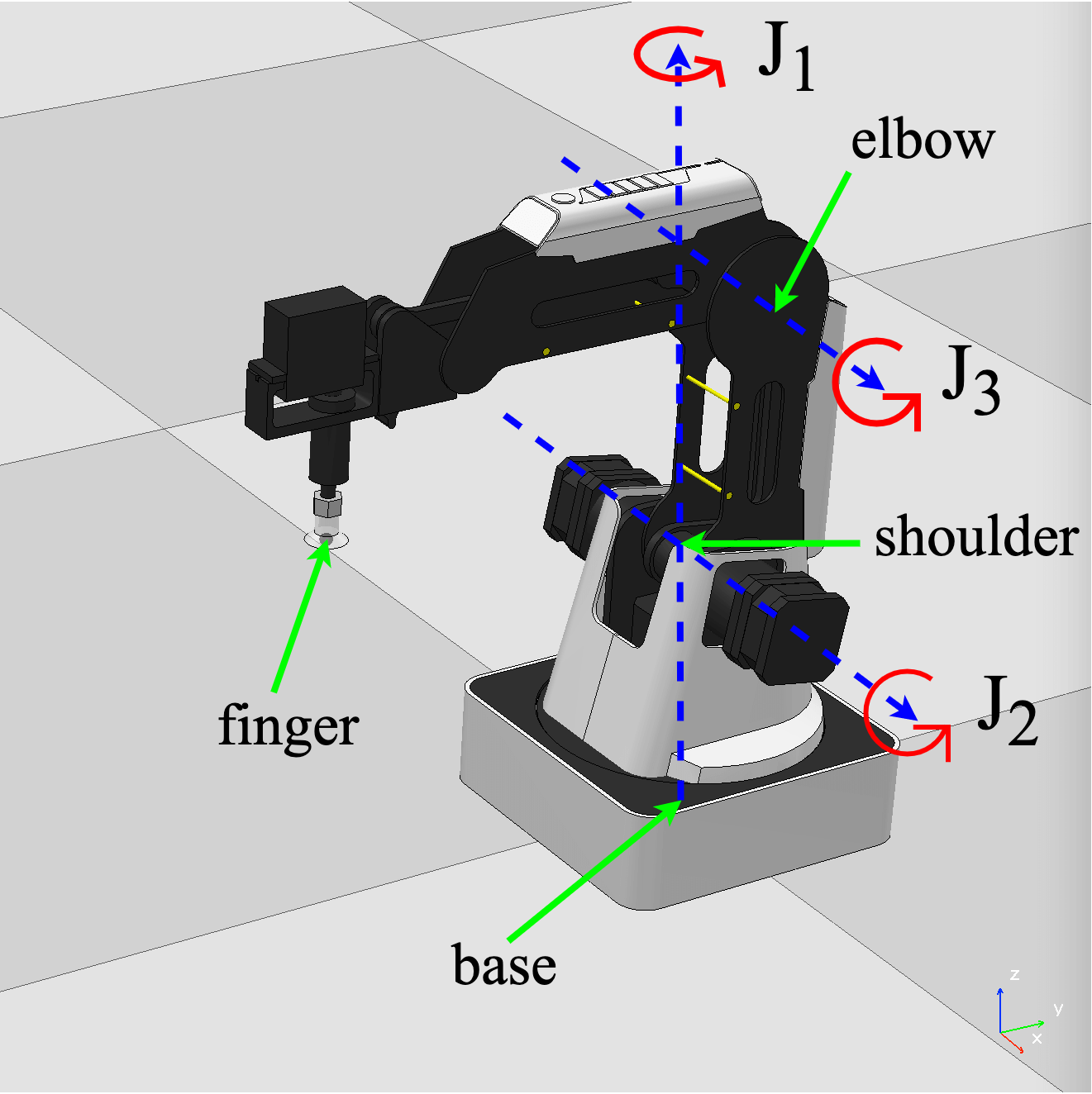}
\caption{The configuration of Dobot Magician, where ``base'', ``shoulder'', and ``elbow'' represent three motors with the joints denoted by $J_1$, $J_2$, and $J_3$, respectively, and ``finger'' represents the end-effector.}
\label{fig:Dobot-env}
\end{figure}

We use Dobot Magician~\citep{islamcartesian} as the learning agent in our experiments, which is a 3-degree of freedom (DOF) robotic arm in both CoppeliaSim simulated environment and real-world industrial applications.
Dobot Magician has three stepper motors to actuate its joints $J_1$, $J_2$, and $J_3$, as shown in Fig.~\ref{fig:Dobot-env}, which can achieve rotation angles within the range of $[-90\degree, 90\degree], [0\degree, 85\degree], [-10\degree, 90\degree]$, respectively. 
Therefore, we consider the action as a 3-dimensional vector clipped to be in the range of $[-1\degree, 1\degree]$, representing the rotation velocity of these three joints.
Being analogous to FetchReach-v1, the observations include the Cartesian positions of the ``elbow'', ``finger'', and the rotation angles of the three joints.

\begin{figure*}[tb]
	\centering
	\subfigure[Task I: the target is dynamic while the single obstacle is not.]{\includegraphics[width=0.96\columnwidth]{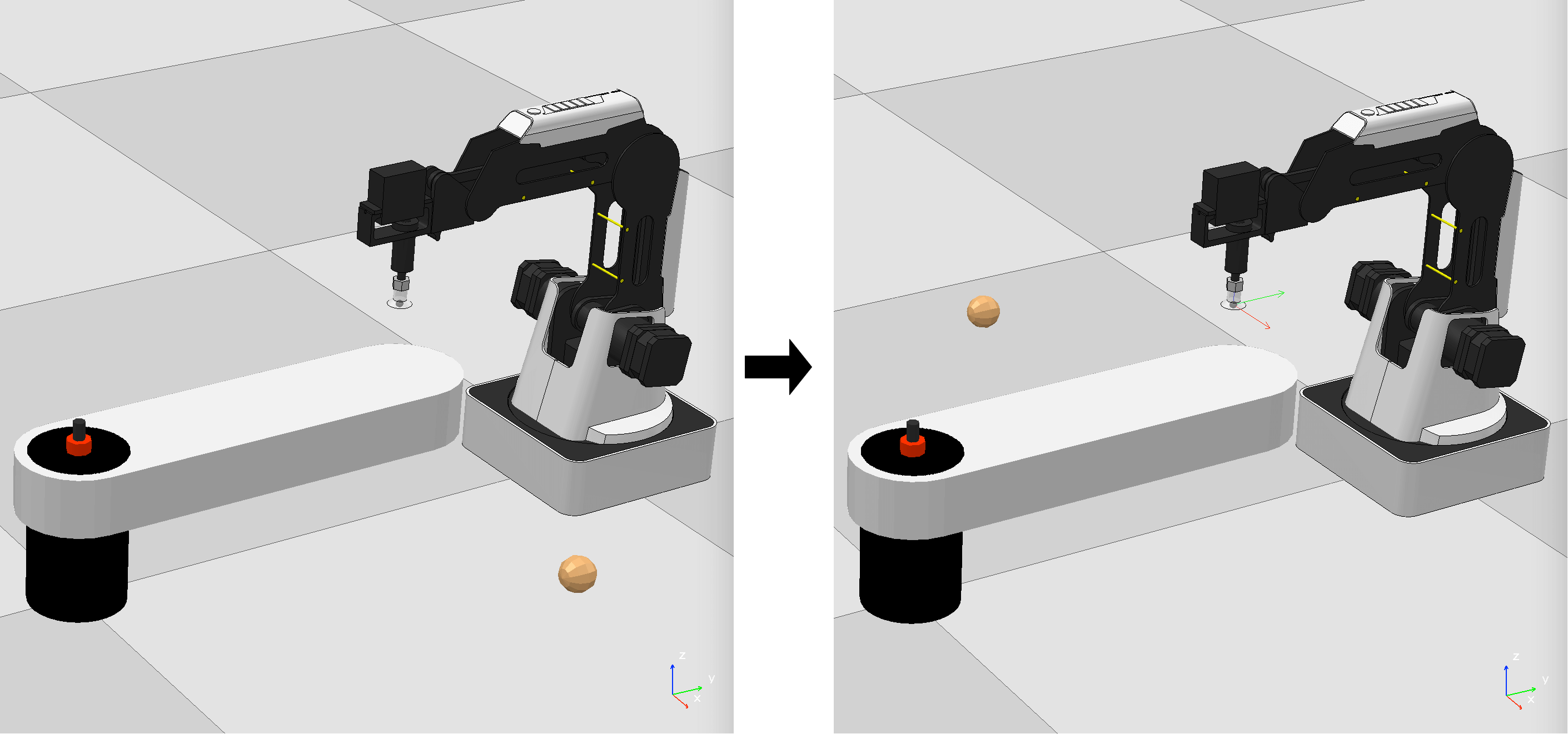}}
	\subfigure[Task II: both the target and single obstacle are dynamic.]{\includegraphics[width=0.96\columnwidth]{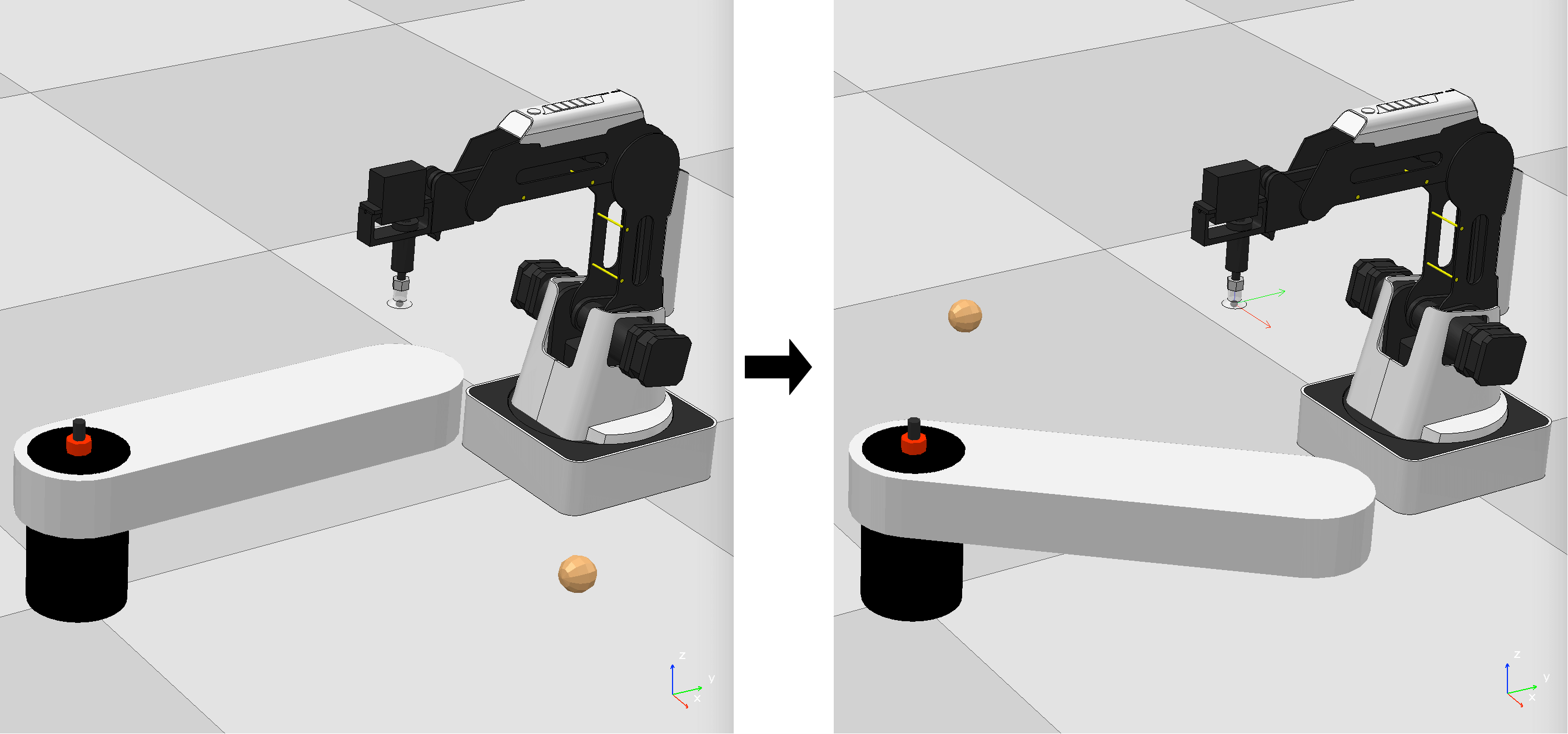}}
	\subfigure[Task III: the multiple obstacles are dynamic while the target is not.]{\includegraphics[width=0.96\columnwidth]{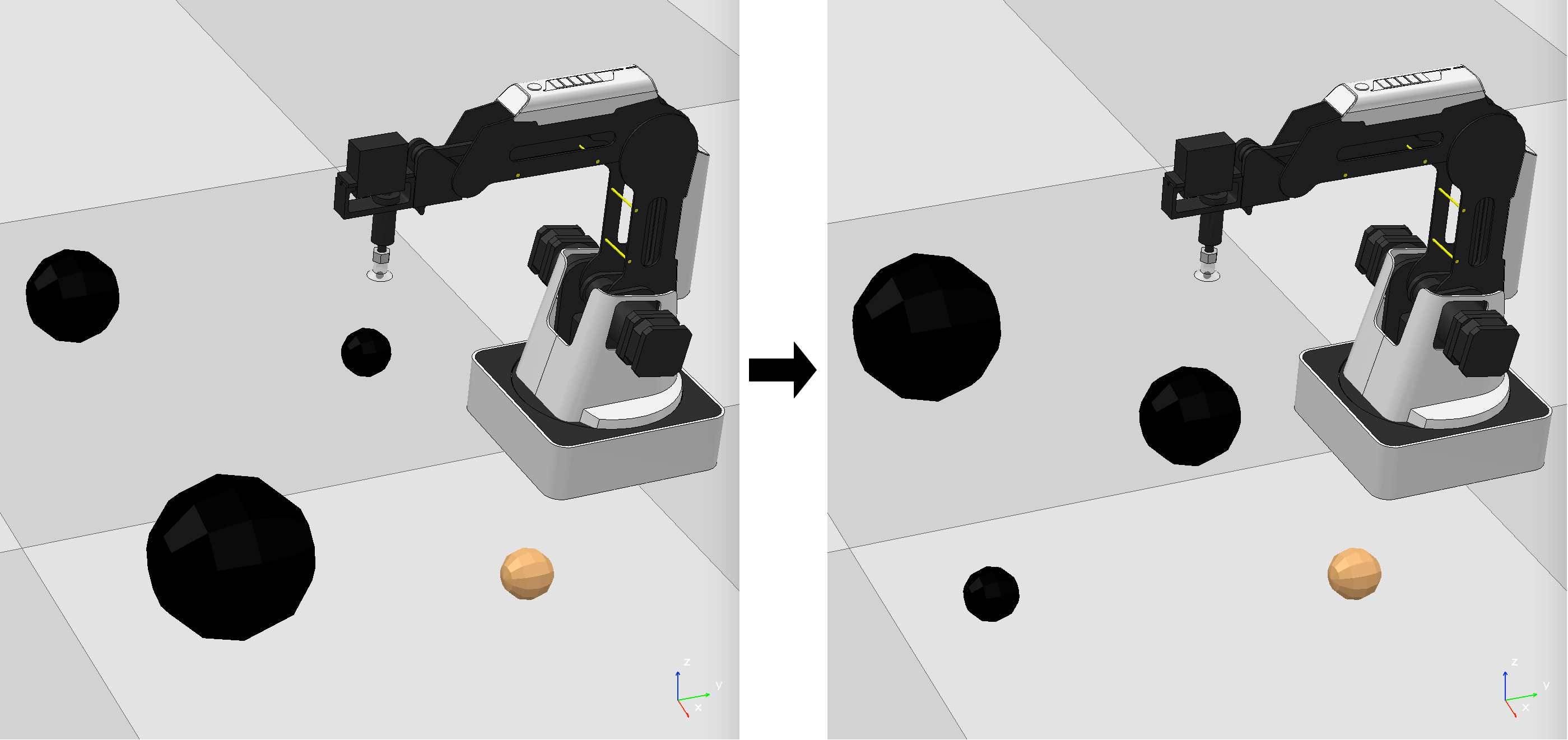}}
	\subfigure[Task IV: both the target and multiple obstacles are dynamic.]{\includegraphics[width=0.96\columnwidth]{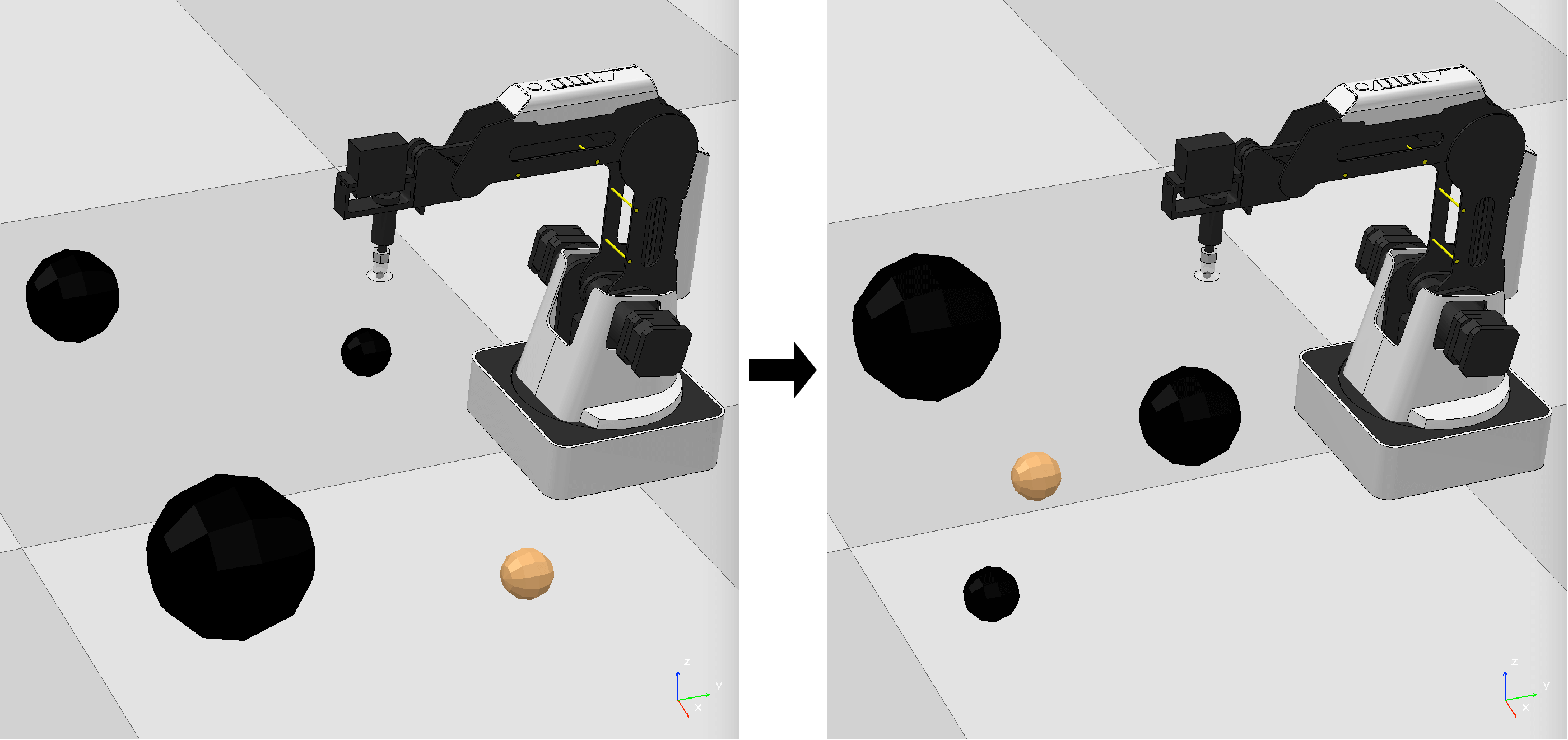}}
	
	\caption{Examples of four goal-conditioned tasks with different dynamics of the target and obstacles.
	The spherical targets in all the tasks are shown in yellow, the single cuboid obstacles in Task I and Task II are shown in white, and the multiple spherical obstacles in Task III and Task IV with various sizes are shown in black.}
	\label{fig:env-tasks}
\end{figure*}

In our experiment, we use Deep Deterministic Policy Gradient (DDPG)~\citep{lillicrap2015continuous} as the base algorithm to evaluate all the investigated methods under the same configurations in the goal-conditioned RL setting.
Specifically, the goal-conditioned policy is approximated by a neural network with two 256-unit hidden layers separated by ReLU nonlinearity as the actor, which maps each state and goal to a deterministic action and updates its weights using Adam optimizer~\citep{kingma2014adam} with a learning rate of $\alpha=3\times 10^{-4}$.
The goal-conditioned Q-function is also approximated by a neural network that maps each state-action pair and goal to a certain value as the critic, which has the same network structure and optimizer as the actor with a learning rate of $\beta=10^{-3}$.
The target networks of actor and critic update their weights by slowly tracking the learned actor and critic with $\tau=0.001$.
We train the DDPG whenever an episode terminates with a minibatch size of 128 for 100 iterations.
In addition, we use a replay buffer size of $10^{6}$, a discount factor of $\gamma=0.99$, and a Gaussian noise $\mathcal{N}(0,0.4)$ to each action for exploration.
For MFRS and DPBA methods, the goal-conditioned potential function $\Phi$ is approximated by a neural network with the same structure and optimizer as the critic network and a learning rate of $\eta=10^{-4}$, which is updated for each timestep in an on-policy manner.
For MFRS in particular, we set the magnet buffer size of $10^{6}$ and the small constant of $\epsilon=10^{-7}$.

For each report unit, we define two performance metrics.
One is the success rate in each learning episode, where success is defined as the agent reaching the target without hitting on the obstacles.
The other is the average timesteps for each trajectory over all learning episodes, defined as $\frac{1}{E}\sum_{k=1}^{E}T_{k}$, where $E$ is the number of learning episodes and $T_{k}$ is the terminal timestep in the $k$th episode.
The former is plotted in figures, and the latter is presented in tables.
Moreover, due to the randomness in the update process of the neural network, we repeat five runs for each policy training by taking different random seeds for all methods, and report the mean regarding the performance metrics.
Our code is available online.
\footnote{\href{https://github.com/Darkness-hy/mfrs}{https://github.com/Darkness-hy/mfrs}}

\subsection{Primary Results}\label{compare-exp}
To evaluate our method in different scenarios, we develop four goal-conditioned tasks with various dynamics configurations of the target and obstacles, on which we implement MFRS and baseline methods.
All four tasks require the agent to move the ``finger'' into a target region without hitting the obstacles in each episode.
Accordingly, we define the original reward function as: $r(s,a,g)=100$ if the ``finger'' successfully reaches the target, $r(s,a,g)=-10$ if the agent hits the obstacle or the floor in CopperliaSim, and $r(s,a,g)=-1$ otherwise.
Each episode starts by sampling the positions of the target and obstacles with the rotation angles of the agent's three joints reset to zero, and terminates whenever the agent reaches the target or at the horizon of $H=1000$.

\subsubsection{Goal-conditioned tasks with different goal dynamics}
In the first two tasks, the obstacle is a cuboid rotator revolving around a fixed cylindrical pedestal, where the length, width, and height of the rotator are $0.1$, $0.4$, $0.05$ along the $x$-axis, $y$--axis, $z$-axis, and the radius, height of the pedestal are $0.04$, $0.13$, respectively. The rotator's axis coincides with the pedestal's axis, and the top surfaces of the rotator and pedestal are on the same plane.
We represent the obstacle position by the center point coordinate of the rotator determined by its rotation angle from an initial position, where the positive direction of rotation is defined to be anticlockwise, with the rotator's initial position being parallel to the $y$-axis.
In the other two tasks with higher complexity, the agent has to handle three dynamic spherical obstacles with a radius of $0.02$, $0.04$, and $0.06$, respectively. We represent the obstacles' positions by the center point coordinates of their own spheres.
Besides, the target is a sphere with a radius of 0.02 in all four tasks where the target position is the center point coordinate.

\begin{itemize}[leftmargin=1em]
\item \textbf{Task I:} As shown in Fig.~\ref{fig:env-tasks}-(a), the dynamic of the goal is determined by randomly changing the target position, which is sampled uniformly from the space both in the reachable area of “finger” and below but not under the lower surface of the single static obstacle.
\item \textbf{Task II:} As shown in Fig.~\ref{fig:env-tasks}-(b), the dynamic goal is created by changing the positions for both the target and the single obstacle. The change of obstacle position can be converted into sampling rotation angles of the rotator, which is defined as a uniform sampling within $[-60\degree, 60\degree]$. The change of target position is consistent with the one in Task I.
\item \textbf{Task III:} As shown in Fig.~\ref{fig:env-tasks}-(c), the goal dynamics is determined by changing the positions of multiple obstacles while holding the target still, and all the target and obstacles are ensured to be not intersectant.
\item \textbf{Task IV:} As shown in Fig.~\ref{fig:env-tasks}-(d), this kind of dynamic goal is created by sampling the positions of all the target and multiple obstacles simultaneously, which is considered the most complex of the four tasks. Also, we need to ensure that all the target and obstacles are not intersectant.
\end{itemize}

\begin{figure*}[tb]
	\centering
	\subfigure[Task I]{\includegraphics[width=0.24\textwidth]{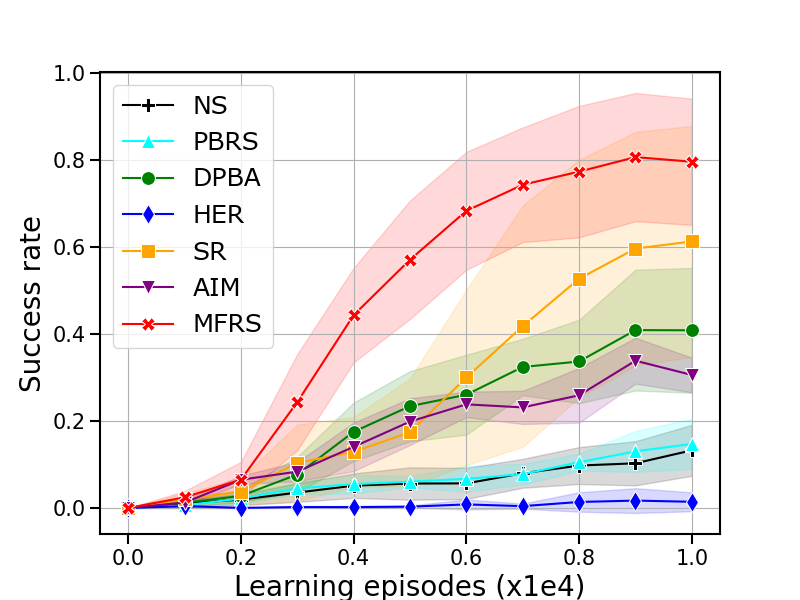}}
	\subfigure[Task II]{\includegraphics[width=0.24\textwidth]{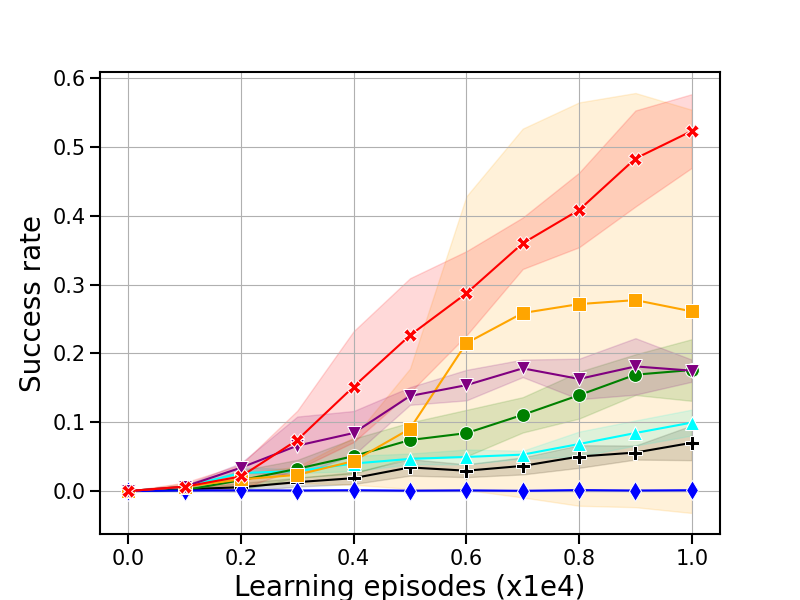}}
	\subfigure[Task III]{\includegraphics[width=0.24\textwidth]{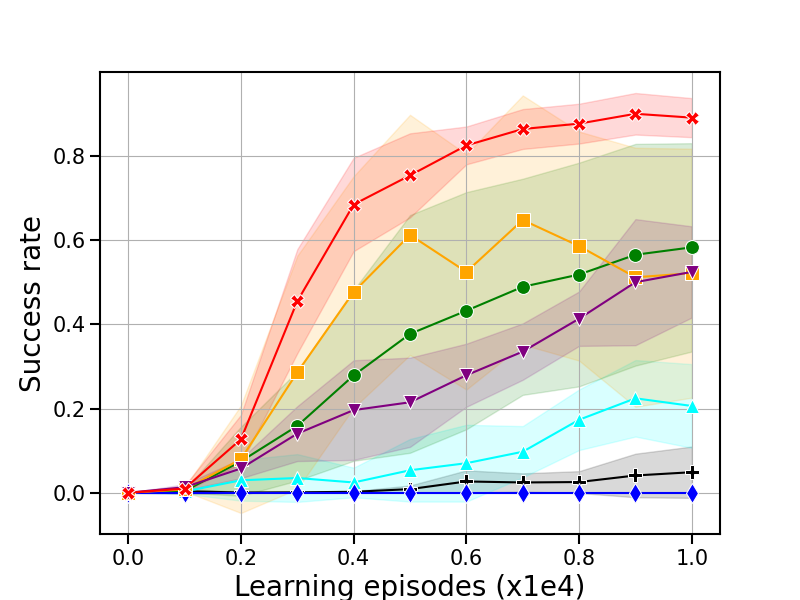}}
	\subfigure[Task IV]{\includegraphics[width=0.24\textwidth]{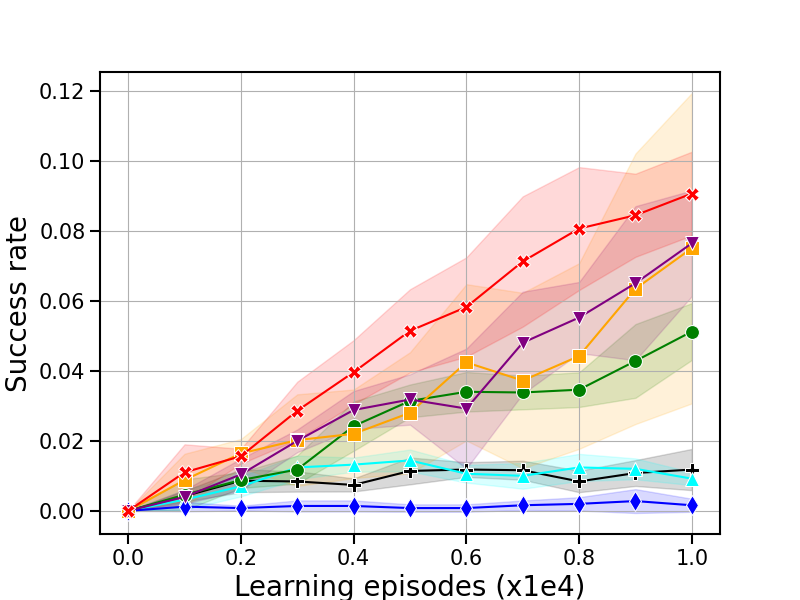}}
	
 	\caption{\textcolor{black}{Success rate in each learning episode of all investigated methods implemented on four different goal-conditioned tasks.}}
	\label{fig:res-success}
\end{figure*}

\begin{figure*}[tb]
	\centering
	\subfigure[Task I]{\includegraphics[width=0.24\textwidth]{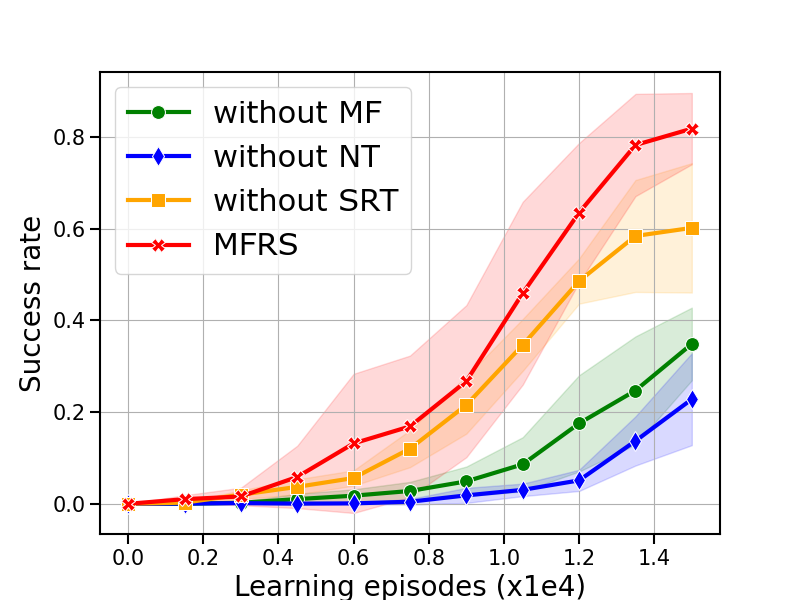}}
	\subfigure[Task II]{\includegraphics[width=0.24\textwidth]{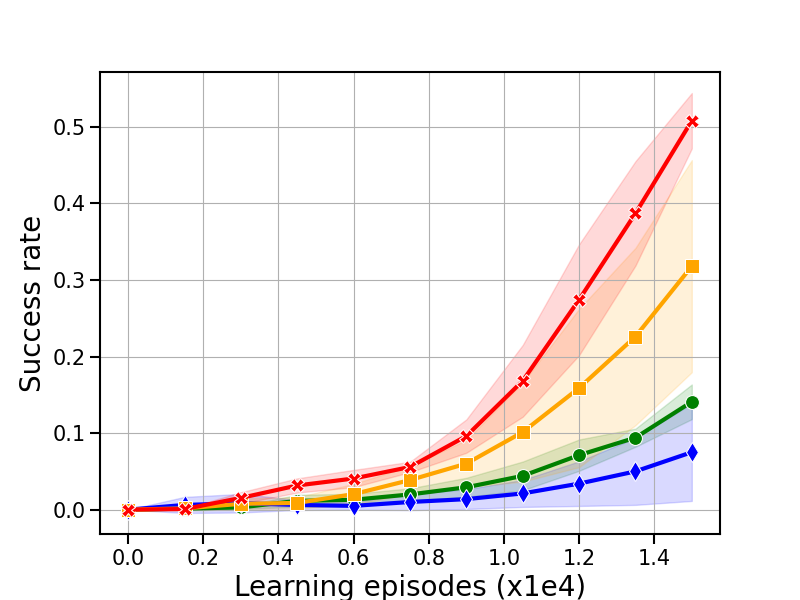}}
	\subfigure[Task III]{\includegraphics[width=0.24\textwidth]{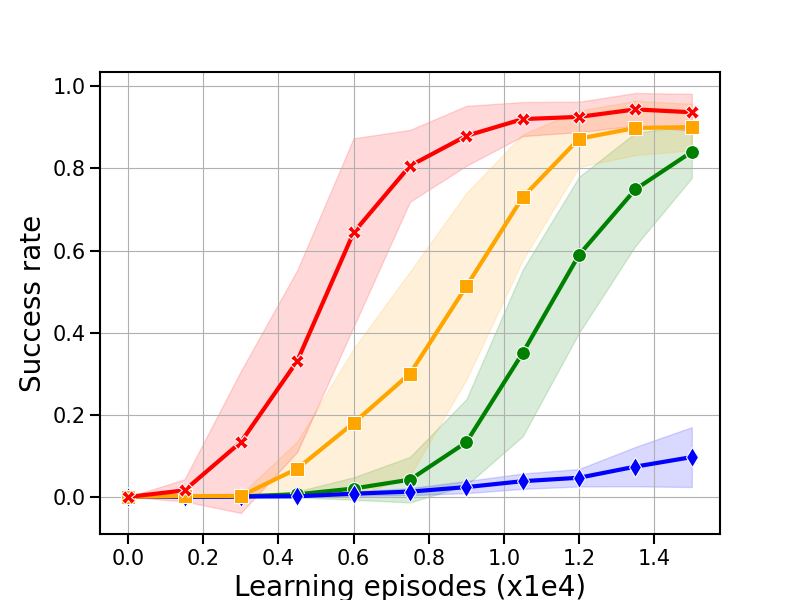}}
	\subfigure[Task IV]{\includegraphics[width=0.24\textwidth]{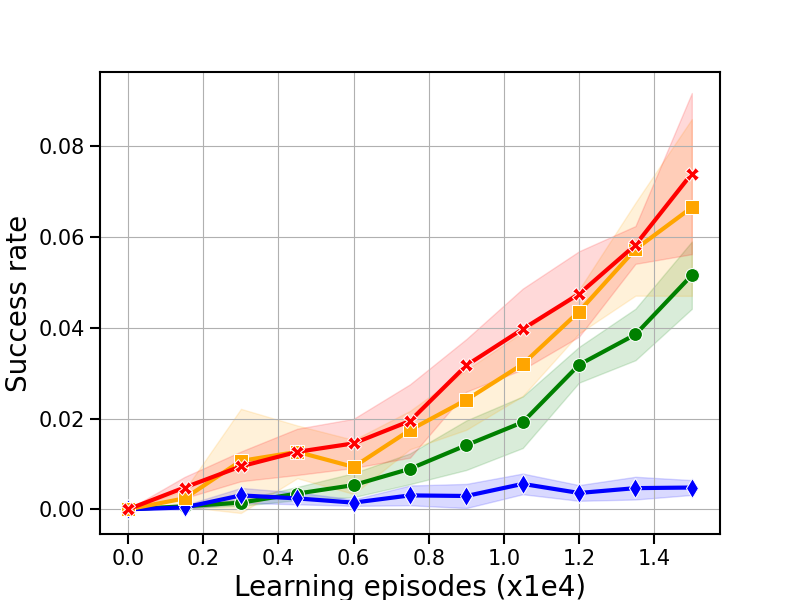}}
	
 	\caption{Success rate in each learning episode of the three variants together with MFRS implemented on four different goal-conditioned tasks.}
	\label{fig:ab-res-success}
\end{figure*}

\subsubsection{Results of MFRS}
We present the experimental results of MFRS and all baselines implemented on the four goal-conditioned tasks.
Fig.~\ref{fig:res-success} shows the success rate in each learning episode, of which the mean across five runs is plotted as the bold line with $90\%$ bootstrapped confidence intervals of the mean painted in the shade.
Table~\ref{table:res-steps} reports the numerical results in terms of average timesteps over 10000 learning episodes of all tested methods. The mean across five runs is presented, and the confidence intervals are corresponding standard errors. The best performance is marked in boldface.

\textcolor{black}{Surprisingly, HER obtains the worst performance and even gets a lower success rate than NS. We conjecture that the relabeling strategy may mislabel the obstacles as the targets when considering the obstacles as a part of the goal.
PBRS achieves a slightly higher success rate than NS since it incorporates additional distance information into the potential function.
DPBA obtains better performance than PBRS, indicating that directly encoding the distance information as the shaping reward and learning a potential function in turn can be more informative and efficient.
AIM has similar performance compared to DPBA and performs better in complex tasks (Task II and Task IV), which is supposed to benefit from the estimated Wasserstein-1 distance between the state visitation distribution and target distribution.
SR performs best among the baseline methods on account of its anti-goals mechanism that encourages exploring and avoids local optima, while it suffers from large confidence intervals since SR fails to consider obstacles and can not guarantee the optimal policy invariance property.
}

\textcolor{black}{In contrast, it can be observed from Fig.~\ref{fig:res-success} that MFRS outperforms all the baseline methods with a considerably larger success rate in all the four tasks, which is supposed to benefit from the explicit and informative optimization landscape provided by the magnetic field intensity.}
The performance gap in terms of success rate at the end of training varies in different tasks. For instance, in Task I, MFRS achieves a success rate of $79.66\%$ at the last learning episode while the baseline methods have only increased the success rate to no more than $61.32\%$, which is $29.91\%$ improved performance at least.
Likewise, the performance in terms of success rate is improved by $100.38\%$ in Task II, $52.85\%$ in Task III and $18.59\%$ in Task IV.
In addition, being different from the research efforts on goal-conditioned RL, our method aims at using reward shaping method to solve sparse reward problems for goal-conditioned RL, which is capable of extending the goal setting to not only dynamic target but also dynamic obstacles while holding the optimal policy invariance property at the same time.
\textcolor{black}{
Moreover, from the large performance gap in terms of success rate between MFRS and AIM, we can deduce that MFRS has the potential to maintain the superiority over other nonlinear reward shaping mechanisms.}

\begin{table}[tb]
\caption{\textcolor{black}{Numerical results in terms of average timesteps over all learning episodes of tested methods in four different goal-conditioned tasks.}}
\centering
\setlength{\tabcolsep}{0.8mm}\renewcommand\arraystretch{1.2}
\begin{tabular}{c|c|c|c|c}
\cmidrule[\heavyrulewidth]{1-5}
Task    & Task I                & Task II               & Task III              & Task IV \\
\hline 
NS      & $939.9\pm 12.7$       & $970.0\pm 4.0$        & $982.8\pm 9.0$        & $990.9\pm 0.9$ \\
PBRS    & $933.1\pm 9.9$        & $952.1\pm 2.7$        & $902.8\pm 29.4$       & $989.9\pm 0.4$ \\
DPBA    & $785.3\pm 31.0$       & $916.5\pm 83.6$       & $660.7\pm 86.3$       & $972.7\pm 1.7$ \\
HER     & $993.4\pm 4.8$        & $999.0\pm 106.2$      & $1000.0\pm $ $0.0$    & $998.6\pm 0.5$ \\
SR      & $684.3\pm 78.3$       & $842.0\pm 78.9$       & $524.5\pm $ $111.2$   & $962.9\pm 9.5$ \\
AIM     & $807.8\pm 7.5$        & $859.9\pm 9.4$        & $731.6\pm $ $17.7$    & $960.1\pm 1.9$ \\
MFRS   	& $\bm{513.1\pm 46.4}$  & $\bm{760.8\pm 62.0}$  & $\bm{380.5\pm 18.7}$  & $\bm{948.6\pm 4.3}$ \\

\cmidrule[\heavyrulewidth]{1-5}
\end{tabular}
\label{table:res-steps}
\end{table}

From Table~\ref{table:res-steps}, it can be obtained that MFRS achieves significantly smaller average timesteps over 10000 learning episodes than all the baselines in all the tasks, which means fewer sampled transitions will be required for the agent to learn the policy. 
\textcolor{black}{Specifically, in Task I, MFRS achieves $513.1$ timesteps in average over all learning episodes while the NS method acquires $939.9$ timesteps.
It indicates that MFRS obtains a $45.4\%$ reduction of sampling transitions.
The decrease of timesteps is $21.6\%$ in Task II, $61.3\%$ in Task III, and $4.3\%$ in Task IV.
Hence, our method successfully improves the sample efficiency of RL in the goal-conditioned setting.
In summary, being consistent with the statement in Section III-A, it is verified that MFRS is able to provide sufficient and conducive information about the complex environments with various dynamics of the target and obstacles, achieving significant performance for addressing the sparse reward problem in different scenarios.}

\subsection{Ablation Study}\label{ablation}
To figure out how the three critical components of our method affect the performance respectively, we perform additional ablation studies using a control variables approach to separate each process apart as the following variants of MFRS.
\begin{itemize}[leftmargin=1em]
\item \textbf{without MF (Magnetic Field):} The shaping reward is calculated based on Euclidean distance, instead of the magnetic field, with normalization techniques and converted into the form of potential-based reward shaping.
\item \textbf{without NT (Normalization Techniques):} The shaping reward is calculated based on the magnetic field without normalization techniques and directly converted into the potential-based form.
\item \textbf{without SRT (Shaping Reward Transformation):} The shaping reward directly follows the form of potential-based reward shaping, with the potential value calculated according to the magnetic field with normalization techniques.
\end{itemize}
The learning performance in terms of success rate per episode of the three variants as well as MFRS is shown in Fig.\ref{fig:ab-res-success}.

First and foremost, the variant \textit{without MF} is compared to MFRS to identify how magnetic field-based shaping reward improves the sample efficiency against distance-based setting.
We can observe that MFRS outperforms the variant \textit{without MF} with a considerably larger success rate per episode in all four tasks.
Specifically, in task III, it takes the variant \textit{without MF} 15000 episodes to achieve a success rate of around $80\%$, while MFRS only needs 8000 episodes to do that.
\textcolor{black}{It verifies that magnetic field-based shaping reward is able to provide a more explicit and informative optimization landscape for policy learning than the distance-based setting. It is consistent with the statement in Section \ref{introduction} that the nonlinear and anisotropic properties of the generated magnetic field provide a sophisticated reward function that carries more accessible and sufficient information about the optimization landscape, thus resulting in a sample-efficient method for goal-conditioned RL.}
 
Next, the variant \textit{without NT} is compared to MFRS to verify the effectiveness of normalization techniques.
Obviously, when taking no account of any normalization in our method, the algorithm performs terribly in all four tasks, which is worse as the task gets more complex.
\textcolor{black}{It verifies the assumption in Section \ref{mag-srf} that if the intensity scales of some magnets are much larger than others, it may amplify the effect of these magnets and give an incorrect reward signal to the agent.
On the other hand, the tremendous intensity value in the very near region of the target will also exacerbate this problem by inducing the agent to slow down the steps to the target, so that it can obtain higher cumulative rewards than directly arriving at the target and terminating the episode.}

Finally, the variant \textit{without SRT} is compared to MFRS for verifying the contribution of Shaping Reward Transformation to the performance of our method.
From Fig.\ref{fig:ab-res-success}, it can be observed that 
MFRS outperforms the variant \textit{without SRT} with improved performance in terms of success rate per episode to some extent, and the performance gap is increasing along the learning process, especially in complex tasks.
\textcolor{black}{It verifies that the magnetic reward is more informative and effective when approximated by the potential-based shaping reward instead of being regarded as the potential function for the guarantee of optimal policy invariance.
On the other hand, one may alternatively select this simplified form of the \textit{without SRT} variant to achieve similar effects of MFRS, which avoids learning a secondary potential function at the cost of some performance loss in practice.}

\section{Application to Real Robot}\label{sim2real}
To further demonstrate the effectiveness of our method in the real-world environment, we evaluate the performance of MFRS and baselines on a physical Dobot Magician robotic arm.
The robot integrates a low-level controller to move the stepper motor of each joint and provides a Dynamic Link Library (DLL) for the user to measure and manipulate the robot's state, including the joint angles and the position of the end-effector. Due to the cost of training on a physical robot, we first train the policies in simulation and deploy them on a real robot without any finetuning. 
\textcolor{black}{The video is available online.}
\footnote{\href{https://hongyuding.wixsite.com/mfrs}{https://hongyuding.wixsite.com/mfrs} \\ \quad (or \href{https://www.bilibili.com/video/BV1784y1z7Bj}{https://www.bilibili.com/video/BV1784y1z7Bj})}

\subsection{Training in Simulation}
From the real-world scenario, we consider a goal-conditioned task with a dynamic target and a dynamic obstacle, as shown in Fig.~\ref{fig:real-dobot}-(b), and build the same simulated environment accordingly using CoppeliaSim as shown in Fig.~\ref{fig:real-dobot}-(a).
Both the target and obstacle are in the shape of cuboid. The target cuboid has the length, width, and height of $0.03$, $0.045$, $0.02$ along the $x$-axis, $y$--axis, $z$-axis in simulation, and $0.038$, $0.047$, $0.12$ for obstacle cuboid, respectively.
The dynamic of the goal is generated by randomly changing the positions of both the target and obstacle in each episode, which are ensured to be not intersectant.
The observations, actions, and reward function are consistent with the ones in \ref{experiments}, and all the hyper-parameters are set to be the same as those in \ref{ablation}.

\begin{figure}[tb]
\centering
\subfigure[Simulation environment]{\includegraphics[width=0.24\textwidth]{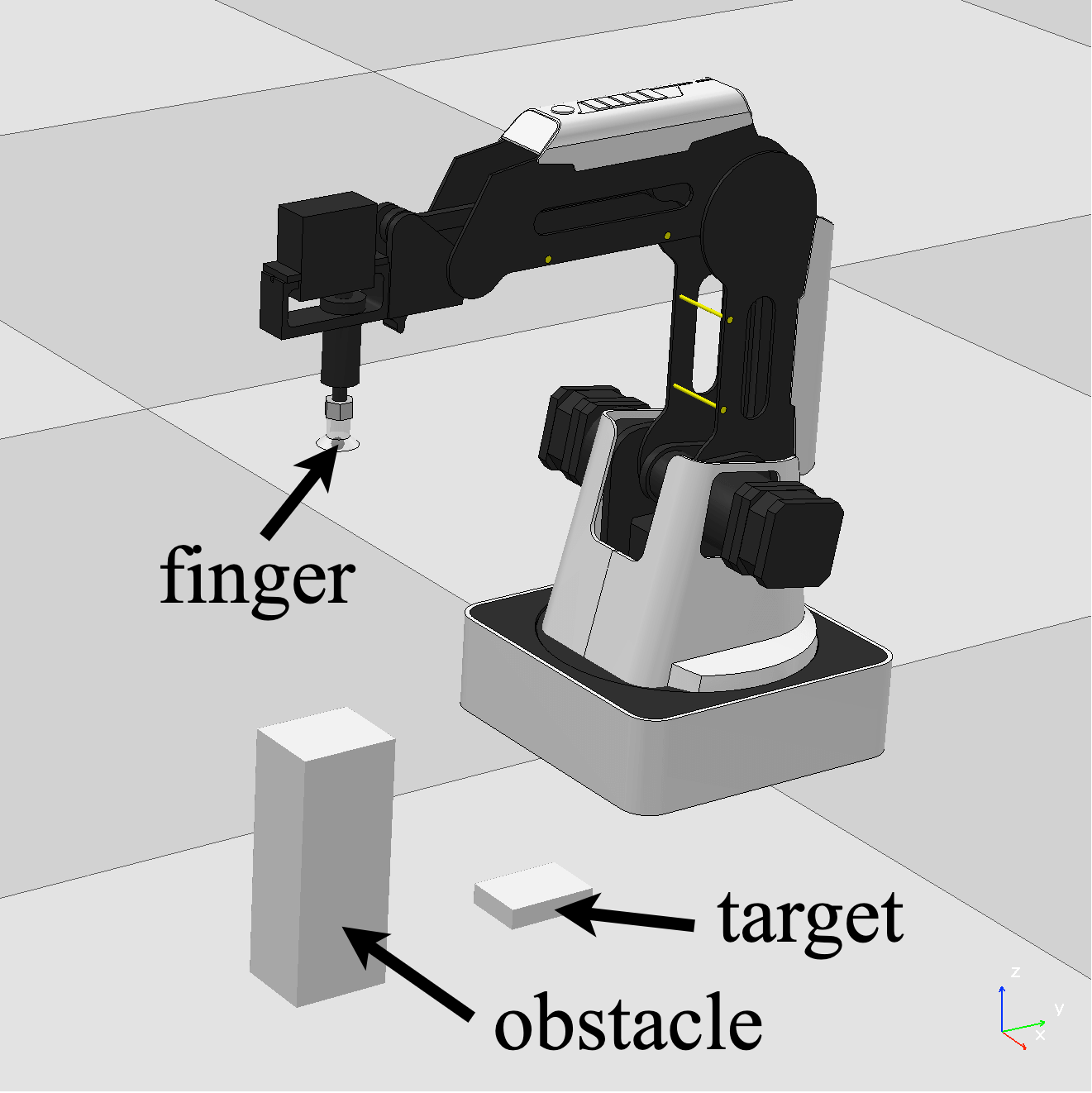}}
\subfigure[Real-world environment]{\includegraphics[width=0.24\textwidth]{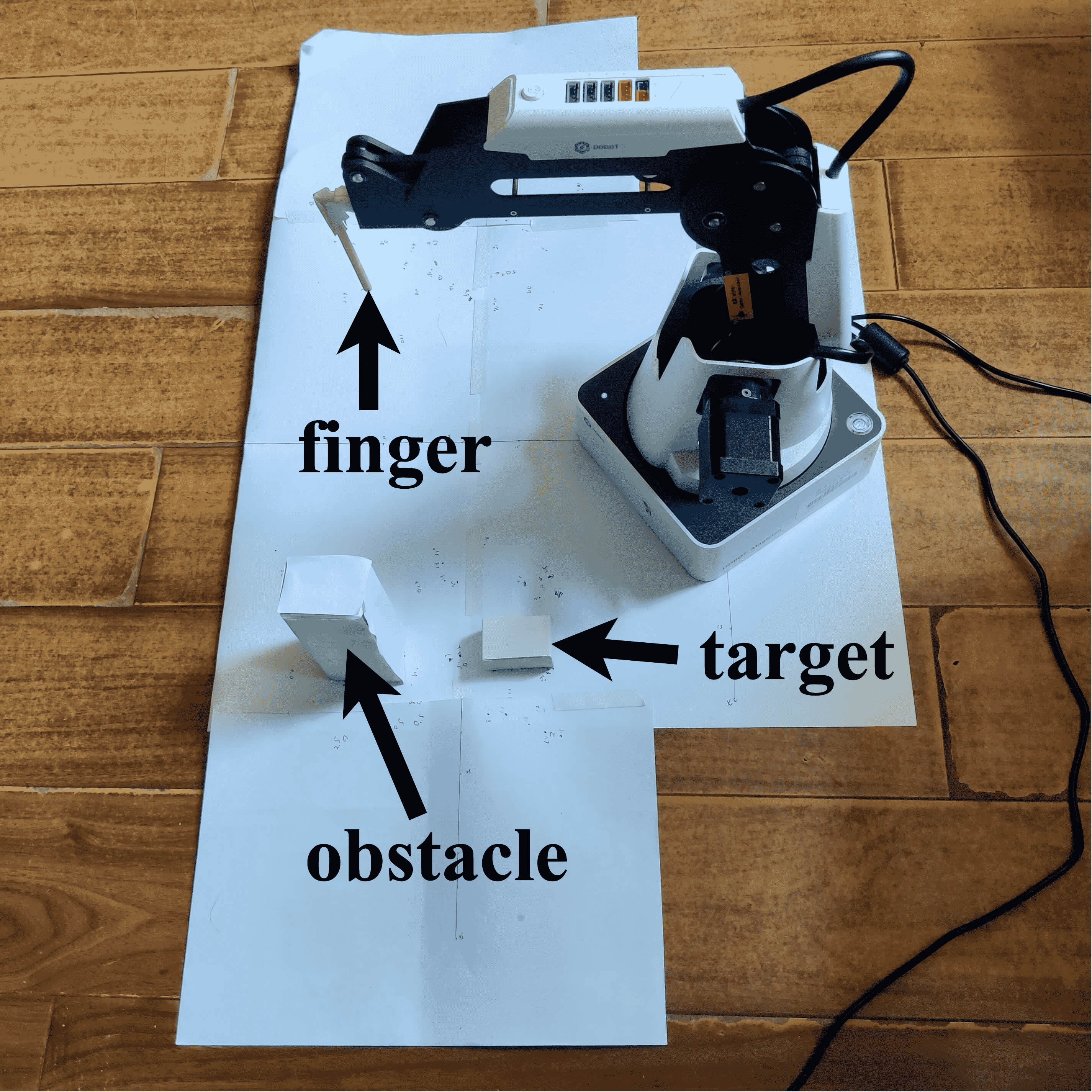}}
\caption{Illustration of the real-world task in (a) simulation environment and (b) real-world environment.}
\label{fig:real-dobot}
\end{figure}

\textcolor{black}{We present the success rate of MFRS and all the baseline methods training on the real-world task in Fig.~\ref{fig:res-real}, of which the mean across $5$ runs is plotted as the bold line with $90\%$ bootstrapped confidence intervals of the mean painted in the shade.
Being consistent with the experimental results in \ref{compare-exp}, MFRS outperforms all other baselines with a significantly larger success rate on the real-world task in simulation.
}

\begin{figure}[tb]
\centering
\includegraphics[width=0.8\columnwidth]{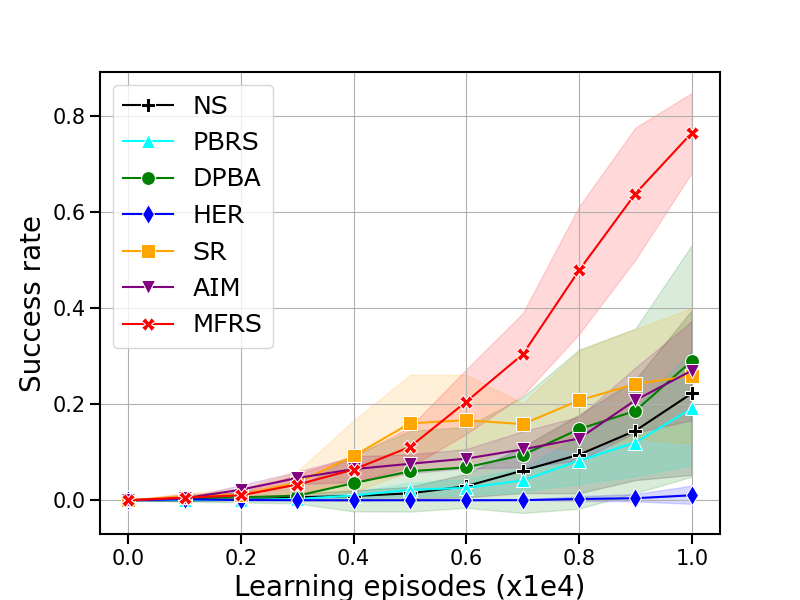}
\caption{\textcolor{black}{Success rate in each learning episode of all investigated methods training on the real-world task in simulation environment.}}
\label{fig:res-real}
\end{figure}

\subsection{Validation in the Real-World}
\textcolor{black}{To figure out whether MFRS performs well and beats the baseline methods on the real robot, we evaluate the success rate and average timesteps over $20$ different testing episodes for each run of each method in the training phase, in which the positions of the dynamic target and obstacle are randomly generated and set accordingly.
That is, for each involved algorithm in real-world experiments, we have conducted $100$ random testing episodes with different settings of the target and obstacle in total.}

In practice, we read the stepper motors' rotation angles and the end-effector's coordinate by calling the DLL provided by the Dobot Magician, which are then concatenated to form the agent's state.
The agent's action in each timestep is obtained by the trained policy, and employed on the real robot as incremental angles to move the three joints using the point-to-point (PTP) command mode of Dobot Magician.
During validation, each testing episode starts with the joint angles reset to zero and manually setting the positions of the target and obstacle in the real-world environment that are sampled randomly in simulation, and terminates whenever the ``finger'' reaches the target, or the robot hits the obstacle or at the horizon of $H=200$.
To prevent unsafe behaviors on the real robot, we restrict the target's sample space to half of the reachable area of the ``finger''.
Computation of the action is done on an external computer, and commands are streamed over the radio at 10Hz using a USB virtual serial port as communication.

\begin{table}[tb]
\caption{\textcolor{black}{Numerical results in terms of success rate and average timesteps over 20 testing episodes of all investigated methods in the real-world environment.}}
\centering
\setlength{\tabcolsep}{5mm}\renewcommand\arraystretch{1.2}
\begin{tabular}{c|c|c}
\cmidrule[\heavyrulewidth]{1-3}
Method  & Success Rate         & Average Timesteps \\
\hline 
NS      & $0.25\pm 0.21$       & $105.27\pm 48.12$ \\
PBRS    & $0.21\pm 0.15$       & $106.81\pm 47.47$ \\
DPBA    & $0.30\pm 0.29$       & $98.42\pm 51.18$ \\
HER     & $0.01\pm 0.02$       & $147.62\pm 64.17$ \\
SR      & $0.10\pm 0.08$       & $96.74\pm 39.64$ \\
AIM     & $0.28\pm 0.07$       & $86.78\pm 6.28$ \\
MFRS   	& $\bm{0.84\pm 0.07}$  & $\bm{75.92\pm 3.46}$ \\
\cmidrule[\heavyrulewidth]{1-3}
\end{tabular}
\label{table:res-real}
\end{table}

\textcolor{black}{We report the numerical results in terms of success rate and average timesteps over 20 testing episodes for each investigated method in Table~\ref{table:res-real}, where the mean across five runs using the corresponding policy trained in simulation is presented, and the confidence intervals are the standard errors.
From Table~\ref{table:res-real}, it can be observed that MFRS achieves a significantly larger success rate and smaller average timesteps compared to all the baselines on the real robots, which is consistent with the results in the simulation experiment.
In addition, our method obtains relatively smaller confidence intervals and standard errors than the baselines. It indicates that MFRS can provide stable learning results when employed on a real robot.
In summary, it is verified that MFRS is able to handle the goal-conditioned tasks in the real-world scenario using the policy trained in the corresponding simulation environment, providing better performance and successfully improving the sample efficiency of the RL algorithm.}

\section{Conclusion}~\label{conclusion}
In this paper, we propose a novel magnetic field-based reward shaping (MFRS) method for goal-conditioned RL tasks, where we consider the dynamic target and obstacles as permanent magnets and build our shaping reward function based on the intensity values of the magnet field generated by these magnets. MFRS is able to provide an explicit and informative optimization landscape for policy learning compared to the distance-based setting.
To evaluate the validity and superiority of MFRS, we use CoppeliaSim to build a simulated 3-D robotic manipulation platform and generate a set of goal-conditioned tasks with various goal dynamics.
Furthermore, we apply MFRS to a physical robot in the real-world environment with the policy trained in simulation.
Experimental results both in simulation and on real robots verify that MFRS significantly improves the sample efficiency in goal-conditioned RL tasks with the dynamic target and obstacles compared to the relevant existing methods.

\textcolor{black}{Our future work will focus on extending MFRS to the scenario with high-dimensional goals using the concept and properties of the magnetic field, and further generalizing to more diversified real-world tasks apart from the area of robotic control.
Another significant direction would be incorporating some perception abilities with the decision-making RL, e.g., equipping the robot with additional sensors to obtain the pose of the dynamic target and obstacles, to fuse MFRS into an end-to-end integrated pipeline for more practical real-world validation of robotic systems.}

\appendix[\textcolor{black}{Proof of the Optimal Policy Invariance of MFRS}]
\begin{theorem}
    Let $M = (\mathcal{S},\mathcal{G},\mathcal{A},\mathcal{T},R,\gamma)$ be the original MDP with the environment reward $R$, and $M' = (\mathcal{S},\mathcal{G},\mathcal{A},\mathcal{T},R',\gamma)$ be the shaped MDP with the shaped reward $R'=R+F$, where the shaping reward function $F$ satisfies Eq.~(\ref{shaping-reward}) with the goal-conditioned potential function $\Phi$ initialized to zero. 
    Let $\pi_{M}^{*}$ and $\pi_{M'}^{*}$ be the optimal policies in $M$ and $M'$, respectively.
    Then, $\pi_{M'}^{*}$ is consistent with $\pi_{M}^{*}$.
\end{theorem}
\begin{proof}
According to UVFA~\citep{schaul2015universal}, the optimal goal-conditioned Q-function in $M$ should be equal to the expectation of long-term cumulative reward as
\begin{equation}
    Q^{*}_{M}(s,a,g) = E\left[\sum_{t=0}^{\infty} \gamma^{t}r_{t}\right].
\end{equation}
Likewise, the optimal goal-conditioned Q-function in $M'$ can be denoted as

\begin{align}
    Q^{*}_{M'}(s,a,g) &= E\left[\sum_{t=0}^{\infty} \gamma^{t}r_{t}'\right] \nonumber\\
    &= E\left[\sum_{t=0}^{\infty} \gamma^{t}(r_{t} + f_{t})\right].
\end{align}

According to Eq.(\ref{shaping-reward}), we have
\begin{align}
     Q^{*}_{M'}(s,a,g) &= E\Bigg[\sum_{t=0}^{\infty} \gamma^{t}(r_{t}+\gamma \Phi_{t+1}(s_{t+1},a_{t+1},g) \nonumber\\
     &\quad - \Phi_{t}(s_t,a_t,g))\Bigg] \nonumber\\
     &= E\left[ \sum_{t=0}^\infty \gamma^{t}r_{t}\right] + 
     E\left[ \sum_{t=1}^\infty \gamma^{t}\Phi_{t}(s_t,a_t,g)\right] \nonumber\\ 
     &\quad - E\left[ \sum_{t=0}^\infty \gamma^{t}\Phi_t(s_t,a_t,g)\right] \nonumber\\
     &= E\left[\sum_{t=0}^\infty \gamma^{t}r_{t}\right] - \Phi_0(s_0,a_0,g).
\end{align}

Thus, we have $Q^{*}_{M'}(s,a,g) = Q^{*}_{M}(s,a,g) - \Phi_0(s_0,a_0,g)$, where $\Phi_0(s_0,a_0,g)$ denotes the initial value of the goal-conditioned potential function.
The policy is obtained by maximizing the value of goal-conditioned Q-function, and hence the optimal policy in $M'$ can be expressed as
\begin{align}
     \pi_{M'}^{*} &= \argmax_{a\in A} Q_{M'}^{*}(s,a,g) \nonumber\\
     &= \argmax_{a\in \mathcal{A}} \left[Q_{M}^{*}(s,a,g) - \Phi_0(s_0,a_0,g) \right].
\end{align}

Therefore, if the value of $\Phi$ is initialized to zero, then we have $\pi_{M'}^{*} = \argmax_{a\in A} Q_{M'}^{*}(s,a,g) = \pi_{M}^{*}$, which demonstrates that $\pi_{M'}^{*}$ is consistent with $\pi_{M}^{*}$ in the goal-conditioned RL setting of our method.
\end{proof}

\footnotesize
\bibliographystyle{myIEEEtranN}
\bibliography{mfrs}

\begin{IEEEbiography}
[{\includegraphics[width=1.0in,height=1.25in,clip,keepaspectratio]{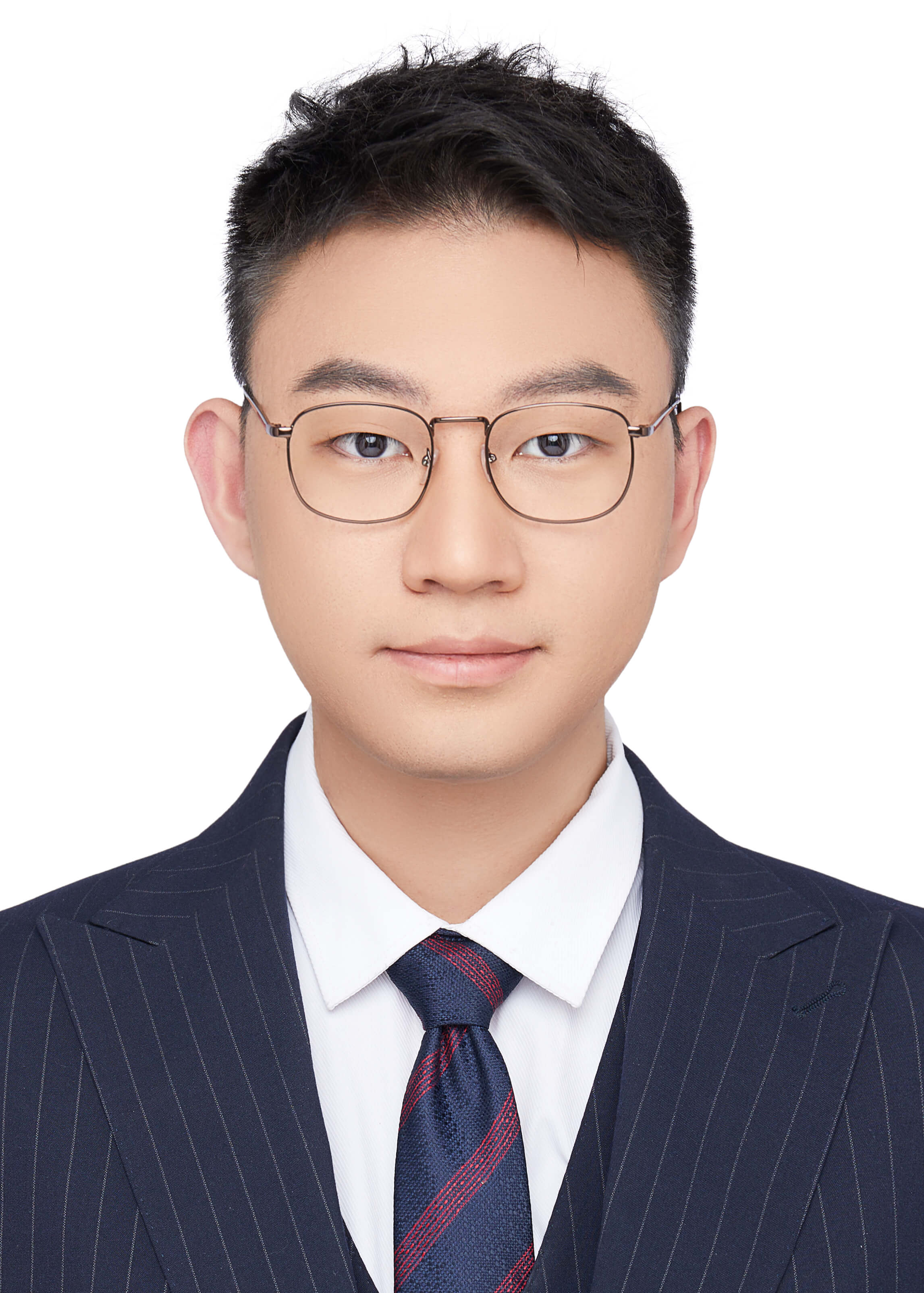}}]{Hongyu Ding}
received the B.E. degree in mechanical engineering from the School of Mechanical and Power Engineering, East China University of Science and Technology, Shanghai, China, in 2021. He is currently pursuing the M.S. degree from the Department of Control Science and Intelligence Engineering, School of Management and Engineering, Nanjing University, Nanjing, China.

His current research interests include reinforcement learning, machine learning, and robotics.
\end{IEEEbiography}

\begin{IEEEbiography}
[{\includegraphics[width=1.0in,height=1.25in,clip,keepaspectratio]{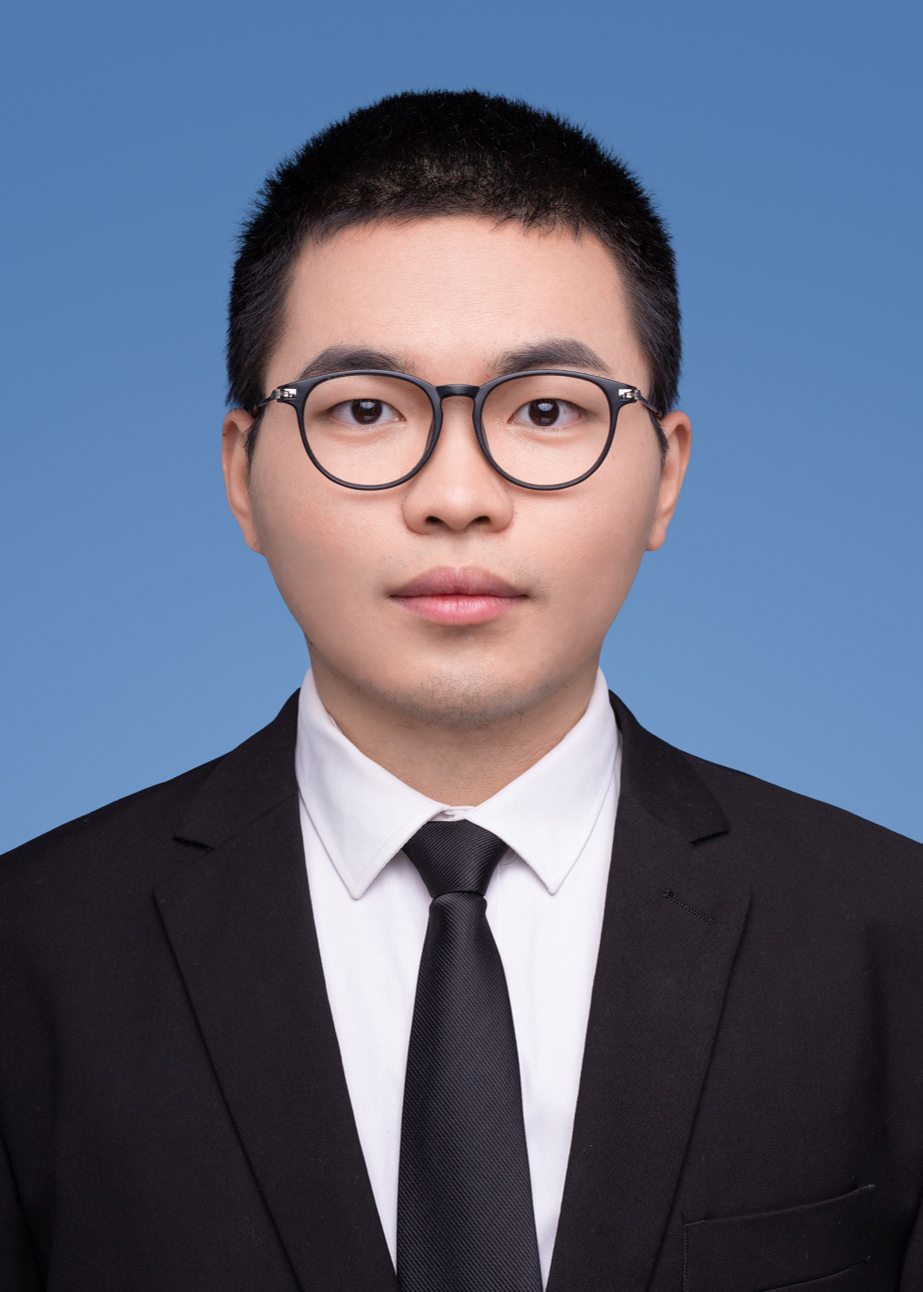}}]{Yuanze Tang}
received the B.E. degree in mechanical engineering, in 2022, from the School of Mechanical and Power Engineering, East China University of Science and Technology, Shanghai, China, where he is currently working toward the Ph.D. degree in digital science and engineering.

His current research interests include machine learning and intelligent life management for engineering equipment.
\end{IEEEbiography}

\begin{IEEEbiography}
[{\includegraphics[width=1.0in,height=1.25in,clip,keepaspectratio]{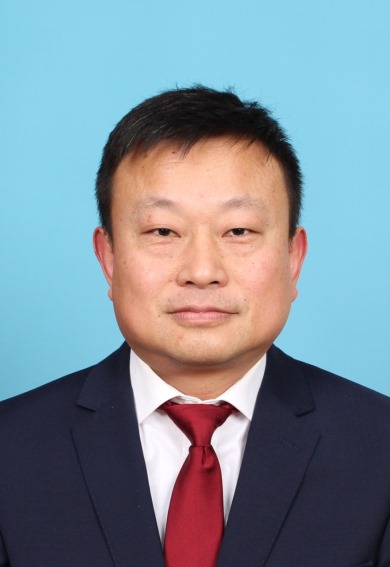}}]{Qing Wu}
received the B.E. and M.E. degrees from the School of Mechanical and Power Engineering, East China University of Science and Technology, Shanghai, China, in 1995 and 2002. He is currently an associate professor in East China University of Science and Technology.

His recent research interests include machine learning, robotics, development and application of embedded system, computer measurement, and control of mechanical and electrical transmission.
\end{IEEEbiography}

\begin{IEEEbiography}
[{\includegraphics[width=1.0in,height=1.25in,clip,keepaspectratio]{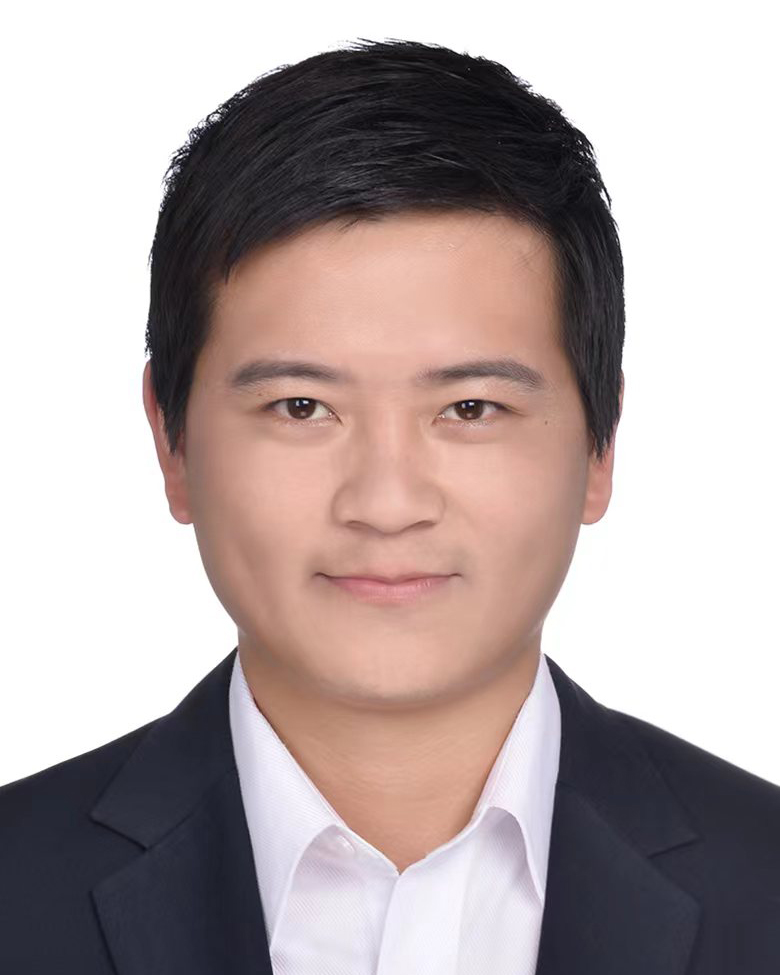}}]{Bo Wang}
(M'16) received the B.Sc. degree in software engineering from Southeast University, Nanjing, China, in 2007, and the M.Sc. and Ph.D. degrees from the Graduate School of Information, Production and Systems, Waseda University, Japan, in 2009 and 2012. He is currently an associate professor with the School of Management and Engineering, Nanjing University, Nanjing, China. He was a Research Assistant of the Global COE Program, Waseda University, Ministry of Education, Culture, Sports, Science and Technology, Japan. He was a Special Research Fellow of the Japan Society for the Promotion of Science (JSPS). He is a Committee Member of Chinese Association of Automation (CAA) Energy Internet Committee, and  a Committee Member of Systems Engineering Society of JiangSu. 

His research interests include, power system planning, renewable generation forecasting, data-driven decision-making, and artificial intelligence algorithms.
\end{IEEEbiography}

\begin{IEEEbiography}[{\includegraphics[width=1.0in,height=1.25in,clip,keepaspectratio]{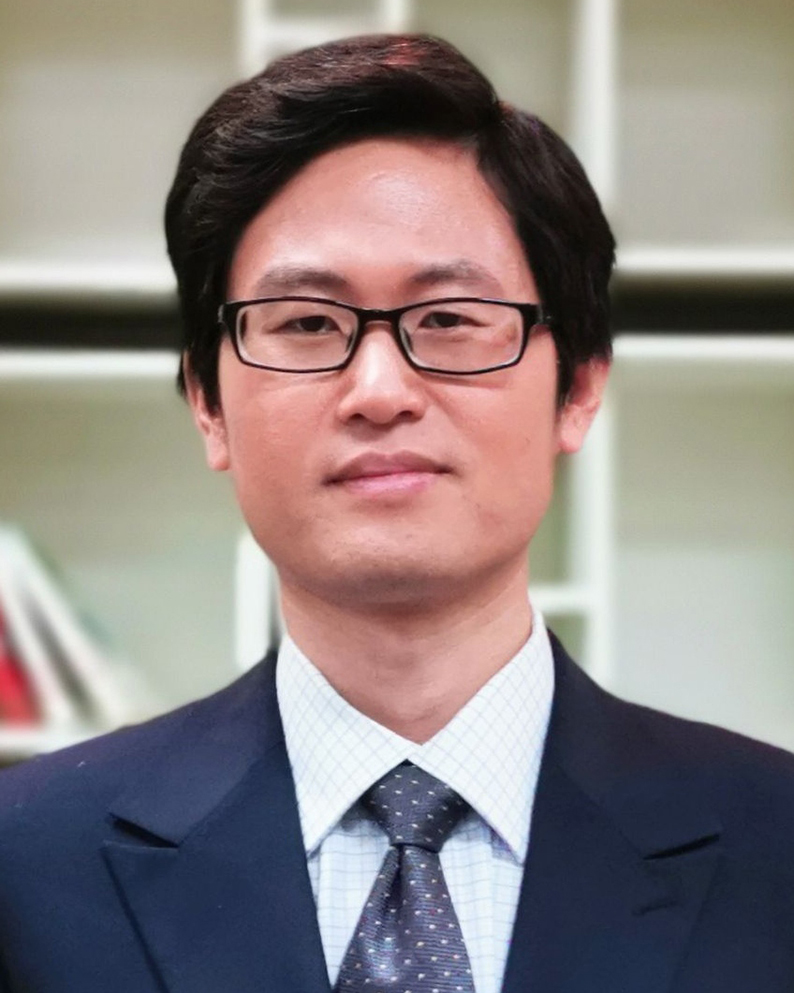}}]{Chunlin Chen}
(S'05-M'06-SM'21) received the B.E. degree in automatic control and Ph.D. degree in control science and engineering from the University of Science and Technology of China, Hefei, China, in 2001 and 2006, respectively. 
He is currently a full professor and the vice dean of School of Management and Engineering, Nanjing University, Nanjing, China. 
He was a visiting scholar at Princeton University, Princeton, USA, from 2012 to 2013. He had visiting positions at the University of New South Wales, Canberra, Australia, and the City University of Hong Kong, Hong Kong, China.

His recent research interests include reinforcement learning, mobile robotics, and quantum control. 
He is the Chair of Technical Committee on Quantum Cybernetics, IEEE Systems, Man and Cybernetics Society.
\end{IEEEbiography}

\begin{IEEEbiography}[{\includegraphics[width=1.0in,height=1.25in,clip,keepaspectratio]{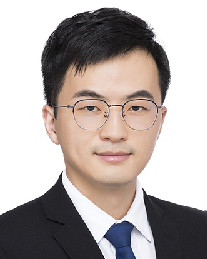}}]{Zhi Wang} (S'19-M'20) received the Ph.D. degree in machine learning from the Department of Systems Engineering and Engineering Management, City University of Hong Kong, Hong Kong, China, in 2019, and the B.E. degree in automation from Nanjing University, Nanjing, China, in 2015. He is currently an Associate Research Fellow with the Department of Control Science and Intelligence Engineering, School of Management and Engineering, Nanjing University, Nanjing, China. He holds visiting positions at the University of New South Wales, Australia and the State Key Laboratory of Management and Control for Complex Systems, Institute of Automation, Chinese Academy of Sciences, China.
	
His current research interests include reinforcement learning, machine learning, and robotics.
He served as the Associate Editor for special sessions of IEEE International Conference on Systems, Man, and Cybernetics 2021 and 2022, and IEEE International Conference on Networking, Sensing, and Control 2020.
\end{IEEEbiography}

\end{document}